%% file: icml_robust.tex
\documentclass[nohyperref]{article}

\usepackage{microtype}
\usepackage{graphicx}
\usepackage{booktabs} 
\usepackage{subfig}
\usepackage{algorithm}
\usepackage{algorithmicx}
\usepackage{algpseudocode}
\usepackage{multicol}
\usepackage{comment}
\usepackage{soul}
\usepackage{microtype} 

\usepackage{hyperref}

\usepackage[accepted]{icml2022}

\usepackage{amsmath}
\usepackage{amssymb}
\usepackage{mathtools}
\usepackage{amsthm}
\usepackage{bbm}

\usepackage[capitalize,noabbrev]{cleveref}

\theoremstyle{plain}
\newtheorem{theorem}{Theorem}[section]

\newtheorem{lemma}[theorem]{Lemma}
\newtheorem{corollary}[theorem]{Corollary}
\theoremstyle{definition}

\newtheorem{assumption}[theorem]{Assumption}
\theoremstyle{remark}

\DeclareMathOperator{\Tr}{Tr}

\DeclareMathOperator*{\argmax}{arg\,max}
\DeclareMathOperator*{\argmin}{arg\,min}

\usepackage[textsize=tiny]{todonotes}

\icmltitlerunning{Investigating Why Contrastive Learning Benefits Robustness against Label Noise}

\begin{document}

\twocolumn[
\icmltitle{
Investigating Why Contrastive Learning Benefits Robustness \\against Label Noise}




\begin{icmlauthorlist}
\icmlauthor{Yihao Xue}{ucla}
\icmlauthor{Kyle Whitecross}{ucla}
\icmlauthor{Baharan Mirzasoleiman}{ucla}
\end{icmlauthorlist}

\icmlaffiliation{ucla}{Department of Computer Science, University of California, Los Angeles, CA 90095, USA}

\icmlcorrespondingauthor{Yihao Xue}{yihaoxue@g.ucla.edu}
\icmlcorrespondingauthor{Kyle Whitecross}{kswhitecross@g.ucla.edu}
\icmlcorrespondingauthor{Baharan Mirzasoleiman}{baharan@cs.ucla.edu}

\icmlkeywords{Machine Learning, ICML}

\vskip 0.3in
]



\printAffiliationsAndNotice{}

\begin{abstract}
\input{abs}    
\end{abstract}
\input{intro}
\input{related}
\input{problem}
\input{method}
\input{experiments}
\input{conclusion}

\section*{Acknowledgements}
This research was supported in part
by Cisco Systems and UCLA-Amazon Science Hub for Humanity and Artificial Intelligence.


\bibliography{icml_robust}
\bibliographystyle{icml2022}

\newpage
\appendix
\onecolumn
\input{appendix}


\end{document}

%% file: abs.tex
Self-supervised Contrastive Learning (CL) has been recently shown to be very effective in preventing deep networks from overfitting noisy labels. Despite its empirical success, the theoretical understanding of the effect of contrastive learning on boosting robustness is very limited. In this work, we rigorously prove that the representation matrix learned by contrastive learning boosts robustness, by having: (i) one prominent singular value corresponding to each sub-class in the data, and significantly smaller remaining singular values; and (ii) {a large alignment between the prominent singular vectors and the clean labels of each sub-class. The above properties enable a linear layer trained on such representations to effectively learn the clean labels without overfitting the noise.} We further show that the low-rank structure of the Jacobian of deep networks pre-trained with contrastive learning allows them to achieve a superior performance initially, when fine-tuned on noisy labels. Finally, we demonstrate that the initial robustness provided by contrastive learning enables robust training methods to achieve state-of-the-art performance under extreme noise levels, e.g., an average of 27.18\% and 15.58\% increase in accuracy on CIFAR-10 and CIFAR-100 with 80\% symmetric noisy labels, and 4.11\% increase in accuracy on WebVision. 

%% file: intro.tex
\section{Introduction}
Large datasets have enabled deep neural networks to achieve a remarkable success in various domains, such as vision and natural language processing \cite{imagenet_cvpr09,floridi2020gpt}. However, this success is highly dependent on the quality of the training labels.
As datasets grow, manual labeling of data becomes prohibitive and the commonly used web-crawling, crowd-sourcing, and automated data labeling techniques result in noisy labels being ubiquitous in large real-world datasets \cite{krishna2016embracing}.
Over-parameterized networks trained with first-order gradient methods can fit any (even random) labeling of the training data \cite{zhang2016understanding}. 
Hence, noisy labels drastically degrade the generalization performance of deep models. To address this, techniques than can robustly learn from noisy labeled data has attracted a lot of attention in recent years \cite{C2D,li2020dividemix,zhang2020distilling,cao2020heteroskedastic,mirzasoleiman2020coresets}.


Classical work on robust learning from noisy labels is mainly focused on estimating the noise transition matrix \cite{goldberger2016training, patrini2017making}, designing robust loss functions \cite{ghosh2017robust, van2015learning, wang2019imae, zhang2018generalized}, correcting noisy labels \cite{ma2018dimensionality, reed2014training, tanaka2018joint,li2020dividemix}, using explicit regularization techniques \cite{cao2020heteroskedastic,zhang2020distilling, zhang2017mixup,liu2020early}, and selecting or reweighting training examples \cite{chen2019understanding, han2018co, jiang2018mentornet, malach2017decoupling, ren2018learning, wang2019imae,mirzasoleiman2020coresets}.
However as the level of noise increases, these techniques become highly ineffective.

Very recently, self-supervised contrastive learning 
has shown a lot of promise in boosting robustness of deep networks against noisy labels. 
Contrastive learning discards all the labels, and learns representations by maximizing agreement between differently augmented views of the same data point via a contrastive loss in the latent space \cite{paper_simclr}. Then a linear layer is trained on the representations with the (potentially noisy) labels in a supervised manner. Empirically, networks trained in this way enjoy a superior degree of robustness against noisy labels \cite{C2D,hendrycks2019usingssl, ghosh2021contrastive}.

Despite its empirical success, the theoretical understanding of the effect of contrastive learning on 
improving robustness of deep networks against noisy labels
is very limited. 
To the best of our knowledge, the only existing theoretical result is on training a binary classifier on pre-trained embeddings obeying a Gaussian distribution \cite{demystifying}. The corresponding theory is, however, derived under very limited assumptions, 
and does not use any properties of self-supervised or contrastive learning.

{In this work, we address the above limitations by theoretically characterizing the beneficial properties of representations obtained by contrastive learning {for enhancing robustness against noisy labels}.
We prove that contrastive learning produces
a representation matrix that has: (i) 
a prominent singular value corresponding to each sub-class in the data, 
and {significantly smaller} remaining singular values; 
and (ii) a large alignment between the prominent singular vectors and the ground-truth labels.
Then we analyze the case where a linear model is trained on the obtained representations with labels that are either perturbed with 
Gaussian noise, or flipped at random to other classes. We show that noise has minimal effect on learning the clean labels and the model can hardly memorize the wrong labels.}


We further show that 
deep networks pre-trained with contrastive learning and fine-tuned on noisy labels can achieve a superior performance initially, before overfitting the noise. 
{
This is attributed to the initial low-rank structure of the Jacobian. Contrastive pre-training produces a Jacobian matrix with a larger gap between the prominent singular values and the remaining smaller ones, compared to a randomly initialized network.  
This gap effectively slows down overfitting at the early phase of training. }
Finally, we demonstrate that the initial robustness provided by contrastive learning can be further leveraged by robust methods to achieve state-of-the-art performance under extreme levels of noise. 
Such methods do not let the low-rank Jacobian matrix to overfit the noise, even after a long number of training iterations.


We conduct extensive experiments on noisy CIFAR-10 and CIFAR-100 \cite{krizhevsky2009learning}, where noisy labels are generated by random flipping the original ones, and the mini Webvision datasets \cite{li2017webvision} which is a benchmark consisting of images crawled from websites, containing real-world noisy labels. We show that contrastive learning enables robust training methods to 
achieve state-of-the-art performance, e.g., 
an average of 27.18\% and 15.58\% increase in accuracy
on CIFAR-10 and CIFAR-100 with 80\% symmetric noisy labels, and 4.11\% increase in accuracy on WebVision.\looseness=-1

%% file: related.tex
\section{Additional Related Work}

\paragraph{Contrastive learning and robustness against noise.}
Recent 
empirical results demonstrated the effectiveness of self-supervised learning in improving robustness of deep models against 
adversarial examples, label corruption and input corruption \cite{hendrycks2019usingssl}.
Contrastive learning has been also shown to 
boost robustness of existing supervised methods \cite{ghosh2021contrastive, C2D} to learn with noisy labels. Notably, \citet{C2D} found a large improvement by combining contrastive learning with two state-of-the-art methods, namely ELR \cite{liu2020early} and DivideMix \cite{li2020dividemix}. 


Despise the recent success of contrastive learning in improving robustness of deep networks, a theoretical explanation is yet to be found. 
Very recently, \citet{demystifying} analyzed the performance of a linear binary classifier 
trained on the embeddings obtained by self-supervised learning. However, their results are based on the assumption that the embeddings follow a Gaussian distribution. 
Nevertheless, the validity of such assumption and its relation to self-supervised learning is not justified.
In contrast, we 
rigorously prove that contrastive learning extracts the underlying sub-class structure from the augmented data distribution and encodes it into the embeddings. This guarantees the robustness of the downstream supervised learning task.

\paragraph{Theoretical works on self-supervised learning.}
A recent line of theoretical works have studied self-supervised learning
\cite{arora2019theoretical,tosh2021contrastive, paper_ssl_contra_loss}. In particular, it is shown that under conditional independence between positive pairs given the label and/or additional latent variables, representations learned by reconstruction-based self-supervised learning algorithms can achieve small errors in the downstream linear classification task \cite{arora2019theoretical,tosh2021contrastive}. More closely related to our work is the recent result of \citet{paper_ssl_contra_loss} that analyzed contrastive learning without assuming conditional independence of positive pairs. Based on the concept of augmentation graph, they showed that spectral decomposition on the augmented distribution leads to embeddings with provable accuracy guarantees under linear probe evaluation. Here, we further leverage the properties of the augmentation graph and provide rigorous robustness guarantees for the performance of linear models trained with representations learned by self-supervised contrastive learning on noisy labels.

%% file: problem.tex
\section{Problem Formulation and Background}\label{sec:formulation}
Suppose we have a dataset $\mathcal{D} = \{(\pmb{x}_i, \pmb{y}_i)\}_{i=1}^{n}$, 
where $(\pmb{x}_i, \pmb{y}_i)$ denotes the $i$-th sample with input $\pmb{x}_i \in \mathbb{R}^d$ and its clean one-hot encoded label $\pmb{y}_i \in \mathbb{R}^K$ corresponding to one of the $K$ classes. For example, for a data point $\pmb{x}_i$ from class $j\in[K]$, we have $\pmb{y}_i=\pmb{e}_j$ where $\textbf{e}_j$ denotes the vector with a 1 in the $j$th coordinate and 0’s elsewhere.
We further assume that there are $\bar{K}\geq K$ sub-classes in the data. 
Sub-classes of a class share the same label, but are distinguishable from each other. 
For example, apple and orange could be two sub-classes of the class fruit.

We assume that for every data point $\pmb{x}_i$, we only observe a noisy version of its label $\hat{\pmb{y}}_i$. The noise $\Delta \pmb{y}_i$ can be either generated from a Gaussian distribution $\Delta \pmb{y}_i=\mathcal{N}(0,\sigma^2 \pmb{I}_n/K)$, or by randomly flipping the label to one of the other classes. For example for a data point $\pmb{x}_i$ whose label is flipped from class $j$ to $k$, we have $\Delta \pmb{y}_i=\pmb{e}_k-\pmb{e}_j$.
We denote by $\pmb{Y},\hat{\pmb{Y}}\in\mathbb{R}^{n\times K}$ the matrices of all the one-hot encoded clean and noisy labels of the training data points. 

We consider the case where the representations are learned with self-supervised contrastive learning, and then a linear layer is trained with the representations on the noisy labels.

\subsection{Self-supervised Contrastive Learning}\label{sec:contrastive}
Self-supervised contrastive learning learns representations of different data points by maximizing agreement between differently augmented views of the same example {and minimizing agreement between differently augmented views of different examples}. This is achieved via a contrastive loss in the latent space, as we discuss below.

\vspace{-3mm}\paragraph{Augmentation graph.}
The augmentations of different data points can be used to construct the \textit{population augmentation graph} \cite{paper_ssl_contra_loss}, 
whose vertices are all the augmented data points in the population distribution, and two vertices are connected with an edge if they are augmentations of the same natural (original) example.
Hence, ground-truth classes naturally form connected sub-graphs. Formally, let $P$ be the distribution of all natural data points (raw inputs without augmentation). For a natural data point $\pmb{x}^*\sim P$, let $\mathcal{A}(\cdot|\pmb{x}^*)$ be the distribution of $\pmb{x}^*$'s augmentations.
For instance, when $\pmb{x}^*$ represents an image, $A(.|\pmb{x}^*)$ can be the distribution of common augmentations \cite{paper_simclr} including Gaussian blur,
color distortion and random cropping.
Then, for an augmented data point $\pmb{x}$,  $\mathcal{A}(\pmb{x}|\pmb{x}^*)$ is the probability of generating $\pmb{x}$ from $\pmb{x}^*$. The edge weights $w_{\pmb{x}_i\pmb{x}_j}=\mathbb{E}_{\pmb{x}^*\sim P}[\mathcal{A}(\pmb{x}_i|\pmb{x}^*)\mathcal{A}(\pmb{x}_j|\pmb{x}^*)]$ can be interpreted as the marginal probability of generating $\pmb{x}_i$ and $\pmb{x}_j$ from a random natural data point.

\vspace{-3mm}\paragraph{Contrastive loss.}
The embeddings produced by contrastive learning can be viewed as a low-rank approximation of the normalized augmentation graph.
Effectively, minimizing a loss that performs spectral decomposition on the population augmentation graph can be succinctly written as a contrastive learning objective $\mathfrak{C}(f)$ on neural network representations \cite{paper_ssl_contra_loss}: 
\begin{align}\label{eq:contrastive}
    \mathfrak{C}(f) \!=\! -2\mathbb{E}_{\pmb{x},\pmb{x}^+}[f(\pmb{x})^{\top}\!\!\!, f(\pmb{x}^+)]+\mathbb{E}_{\pmb{x},\pmb{x}^-}\![\big((f(\pmb{x})^{\top}\!\!\!,f(\pmb{x}^-)\big)^2],
\end{align}
where $f(\pmb{x})\in\mathbb{R}^p$ is the 
neural network representation for an input $\pmb{x}$, and $\pmb{x},\pmb{x}^+$ are drawn from the augmentations of the same natural data point, and $\pmb{x},\pmb{x}^-$ are two augmentations generated independently either from the same data point or two different data points. 
The above loss function is similar to many standard contrastive loss functions \cite{oord2018representation,sohn2016improved,wu2018unsupervised}, including SimCLR \cite{paper_simclr} that we will use in our experiments.
Minimizing this objective leads to representations with provable accuracy guarantees under linear probe evaluation. 
We use $f_{\min}$ to denote the minimizer, i.e., $f_{\min} = \argmin_f \mathfrak{C}(f)$.

\subsection{Training the Linear Head with Label Noise} 
Here, we introduce the notations for training a linear classifier on the representations learned by contrastive learning, based on which we perform theoretical analysis. 
In Section \ref{sec: training_the_whole}, we discuss how our idea can be extended to understand the performance of fine-tuning all the layers of the neural network.\looseness=-1

{We assume the representations are given by $f_{\min}$, the global minimizer of the contrastive loss. In practice, this is easier to be achieved by larger networks trained for longer \cite{paper_simclr}. In Section \ref{sec:exp_cifar}, we confirm superior robustness of representations learned by larger networks against noisy labels by our experiments.
Given 
a matrix $\pmb{F}\in \mathbb{R}^{n\times p}$ where each row $\pmb{F}_i=f_{\min}(\pmb{x}_i)^{\top}$ is the learned representation of a data point $\pmb{x}_i$, we consider the downstream task of training a linear model, parameterized by $\pmb{W}\in \mathbb{R}^{p\times K}$, to minimize the MSE loss with $l_2$ regularization with parameter $\beta$ 
\begin{equation}
\label{eq:objective}
    \min_{\pmb{W}\in \mathbb{R}^{p\times K}} \|\hat{\pmb{Y}}-\pmb{F}\pmb{W}\|_F^2 + \beta\|\pmb{W}\|.
\end{equation}
Let $\hat{\pmb{W}}^*$ denote the solution that has the following closed-form expression
\begin{align}
    \label{eq:closedform_w}
    \hat{\pmb{W}}^* = (\pmb{F}^{\top} \pmb{F}+\beta\pmb{I})^{-1}\pmb{F}^{\top}\hat{\pmb{Y}}.
\end{align}
While we use MSE in our analysis, we empirically show that our results hold for other losses, such as cross-entropy.} \looseness=-1

%% file: method.tex
\section{Contrastive learning Boosts Robustness}
In this section we first show that training a linear head on representations learned by contrastive learning is provably robust to label noise. Then we {look into the phenomenon} that fine-tuning the deep network pre-trained by contrastive learning achieves a superior performance at early phase of training.
{
Finally, we discuss how the initial robustness provided by
contrastive learning boosts robust training methods, and corroborate this with extensive experiments in Section \ref{sec: experiments}.} \looseness=-1

\begin{figure*}[t]
    \centering
    \includegraphics[width=0.268\textwidth]{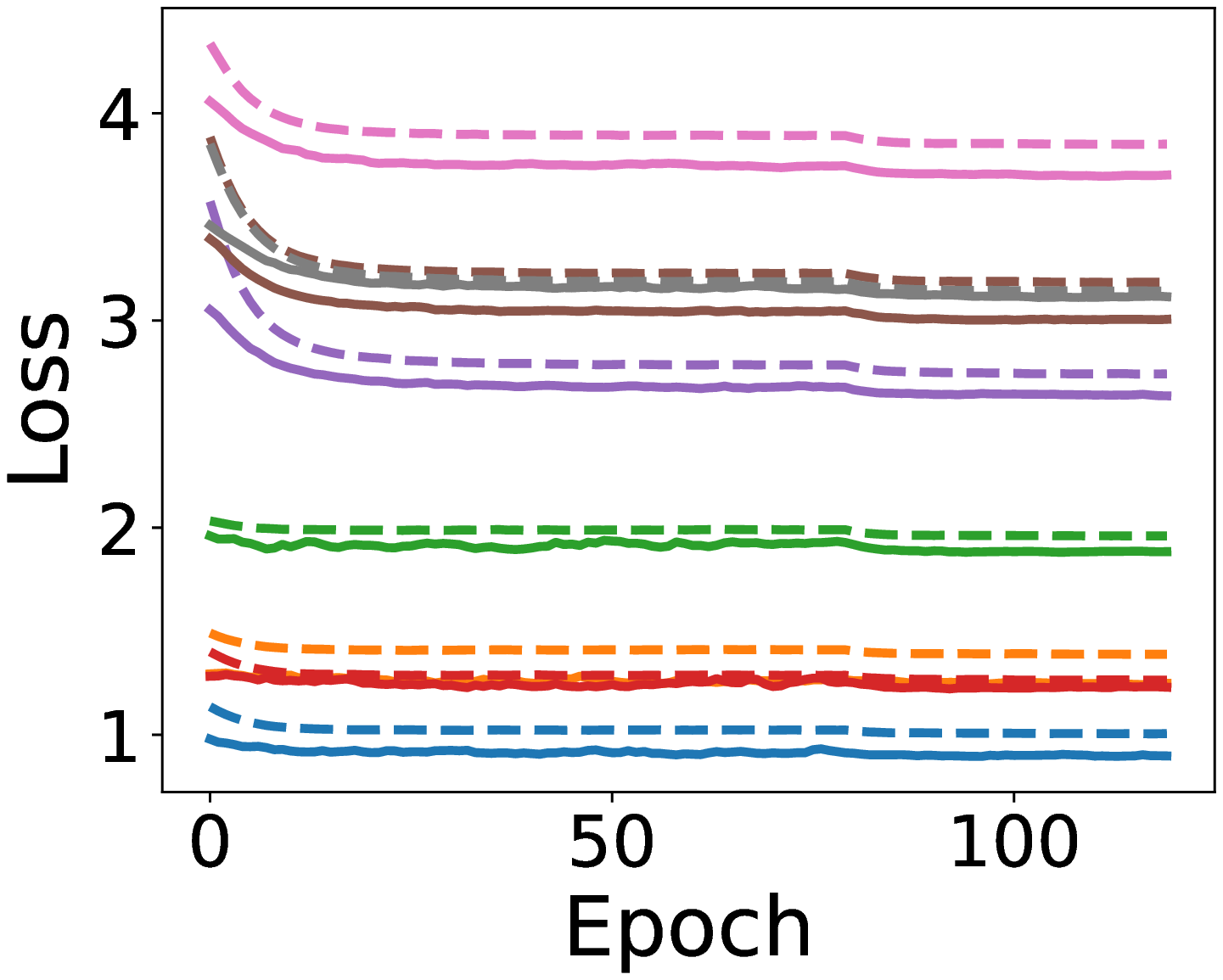}
    \hspace{.6cm}
    \includegraphics[width=0.557\textwidth]{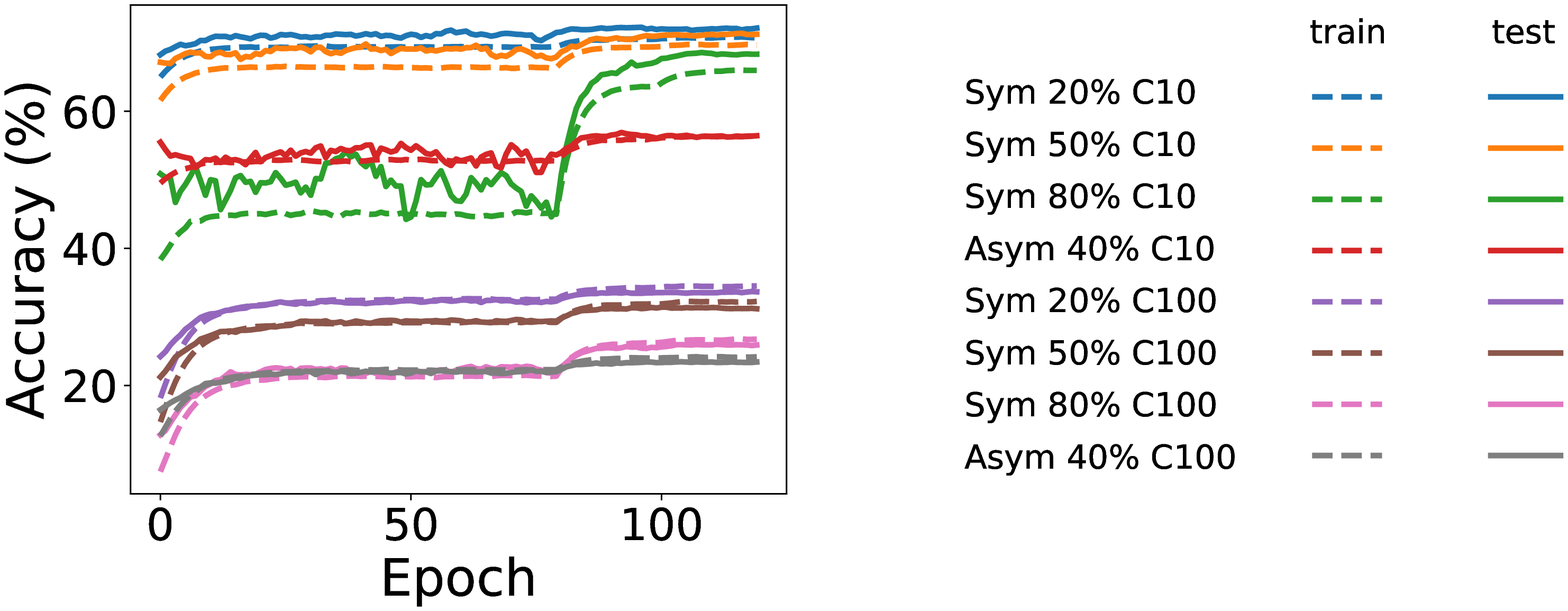}
    \vspace{-.2cm}
    \caption{Training accuracy w.r.t. ground-truth labels and test accuracy of a linear classifier trained on representations learned by contrastive learning (SimCLR). Experiments are conducted on CIFAR-10 (C10) and CIFAR-100 (C100) under different noise levels. Dashed lines show loss and accuracy on training set w.r.t. ground-truth labels and solid lines show test loss and accuracy.}
    \label{fig:train_acc_gt}
\end{figure*}

\subsection{Provable Robustness of the Linear Head}

{
To understand the robustness provided by contrastive learning, 
we assume
certain properties of the augmentation graph
and analyze the low-rank structure of the the resulting representation matrix.}
In particular, we utilize the following natural assumptions that formalize the following two properties on the data augmentation: (1) the augmented examples of one sub-class are similar to each other; and (2) the augmented examples of one sub-class are different from the augmented examples of other sub-classes.

\begin{assumption}[\textbf{Compact sub-class structure}]\label{assump: ratio} For a triple of augmented examples $\pmb{x}_j$, $\pmb{x}_s$ and $\pmb{x}_t$ from the same sub-class, 
the marginal probability of $\pmb{x}_s$, $\pmb{x}_j$ being generated from a natural data point is close to that of $\pmb{x}_t$, $\pmb{x}_j$. Formally, 
we have 
${w_{\pmb{x}_s \pmb{x}_j}}/{w_{\pmb{x}_t \pmb{x}_j}}\in[\frac{1}{1+\delta}, 1+\delta]$, for small $\delta\in[0,1)$.
\end{assumption}
 
\begin{assumption}[\textbf{Distinguishable sub-class structure}]\label{assump: block_with_off_diag} For two pairs of augmentated examples $(\pmb{x}_i, \pmb{x}_j)$ and $(\pmb{x}_s, \pmb{x}_t)$ where $\pmb{x}_i$, $\pmb{x}_j$ are from different sub-classes and $\pmb{x}_s$, $\pmb{x}_t$ are from the same sub-class, 
the marginal probability of $\pmb{x}_i$, $\pmb{x}_j$ being generated from a natural data is much smaller than that of $\pmb{x}_s$, $\pmb{x}_t$. Formally, 
we have 
$w_{\pmb{x}_i \pmb{x}_j}/{w_{\pmb{x}_s \pmb{x}_t}} \leq \xi$, for small $\xi\in[0,\! 1)$. 
\end{assumption}
The above assumptions result in an augmentation graph where augmented data points from different subclasses form nearly disconnected subgraphs with similar edge weights.
In particular for $\xi=0$, we get diconnected subgraph structure.


\subsubsection{\!\!Desirable Properties \!of Cl \!Representations}\label{sec:desirable}
{
The key to our analysis is that, based on compact and distinguishable sub-class structure assumptions \ref{assump: ratio}, \ref{assump: block_with_off_diag}, contrastive learning produces a low-rank representation matrix $\pmb{F}$ that captures the sub-class structure. More formally, the representation matrix has $\bar{K}$ singular vectors that align well with the ground-truth labels, and the corresponding $\bar{K}$ singular values are significant larger than the other singular values.
The following theorem is a summary of Lemmas \ref{lemma:largest_egvalue} \ref{lemma: perron_vec} \ref{lemma: perturb_eigenvalue} \ref{lemma: perturb_proj} and Corollary \ref{corollary:sum_smallest} in the Appendix which details the desirable properties of the representation matrix.} 

\begin{theorem}\label{theo:informal}
Having $\bar{K}$ compact and distinguishable sub-classes in the data,
the representation matrix $\pmb{F}$ learned by contrastive learning has $\bar{K}$ prominent singular values of magnitude $\mathcal{O}(1)$. At the same time, the sum of the remaining singular values is significantly smaller, i.e., $\mathcal{O}(\sqrt{\delta}+\xi)$. Furthermore, the most prominent $\bar{K}$ singular vectors and the ground-truth labels has a $\mathcal{O}(1)$ alignment, measured by the normalized projection of the clean labels $~\pmb{Y}$ onto the span of the singular vectors. 
\end{theorem}
{
Intuitively, the above three properties of the representation matrix affect the downstream training in the following sense:
{(1) the magnitude of largest singular values} determines the speed at which the model evolves as well as the extent to which the model can fit the training data;
{(2) the alignment between prominent singular vectors and clean labels} indicates whether the model evolves in the right direction; and
{(3) the magnitude of smaller singular values} dictates the amount of overfitting. As a result, Theorem \ref{theo:informal} implies that the model trained on such representation learns mainly the correct information from the training data, which we formally show in Theorems \ref{theorem: gaussian_loss} and \ref{theorem: flipping_accuracy}.
}

\subsubsection{{\!\!Training performance w.r.t. Ground-truth Labels 
Reflects Robustness}}
{To simplify the theoretical analysis, instead of studying the generalization performance (usually measured by the expected loss over the data distribution), we will examine the loss and accuracy on the training data w.r.t. ground-truth labels. This strongly correlates with the test accuracy, especially under large noise. We empirically confirm this correlation in Figure \ref{fig:train_acc_gt}, where the dashed lines show training loss and training accuracy w.r.t. ground-truth labels, and solid lines show test loss and test accuracy.} We clearly see the high correlation between training and performance, 
in particular under significant levels of label noise.

\subsubsection{Gaussian Label Noise}
We first consider the case where label noise is generated from a Gaussian distribution.
Formally, $\hat{\pmb{Y}}=\pmb{Y}+\Delta \pmb{Y}$, where ${\pmb{Y}}$ is the clean label matrix containing all the one-hot encoded labels, and $\Delta \pmb{Y}$ is the label noise matrix, where each column drawn independently from $\mathcal{N} (0, \sigma^2\mathbf{I}_n/{K})$. 
{We consider this setting first, as it provides the most convenient way to analyze robustness.}
{Here, our analysis mainly aims at breaking down the effect of 
label perturbations on training dynamics, in terms of bias and variance.
This could provide theoretical insights into the benefits of contrastive learning for boosting robustness.} 


The following theorem bounds the expected error on training data w.r.t. ground-truth labels,
and shows how contrastive learning exploits the augmented sub-class structure to improve robustness.\looseness=-1 

\begin{theorem}\label{theorem: gaussian_loss}
For a dataset of size $n$ with $K$ classes, $\bar{K}$ {balanced} compact and distinguishable sub-classes (\textit{c.f.} assumptions \ref{assump: block_with_off_diag}, \ref{assump: ratio}) and labels corrupted with Gaussian noise $\mathcal{N} (0, \sigma^2\mathbf{I}_n/{K})$, a linear model trained by minimizing the objective in Eq. \eqref{eq:objective} with the representations 
obtained by minimizing contrastive loss in Eq. \eqref{eq:contrastive} has the following expected error on the training set w.r.t. the ground-truth labels $\pmb{Y}$:
\begin{align}\label{eq: upper_bound_error}
    &\mathbb{E}_{\Delta \pmb{Y}}\frac{1}{n}\| \pmb{Y} -\pmb{F}\hat{\pmb{W}}^* \|_F^2 \\
    \nonumber
    \leq & \underbrace{(\frac{\beta}{\beta+1})^2 + \mathcal{O}(\delta+\xi)}_{\textbf{bias}^2} + \underbrace{\sigma^2\frac{\bar{K}}{n}(\frac{1}{\beta+1})^2 + \sigma^2\mathcal{O}(\frac{\sqrt{\delta}+\xi}{\beta})}_{\textbf{variance}}.
\end{align}
\end{theorem}
{We note that the above results can be easily extended to imbalanced sub-class structure. 

The proof can be found in the Appendix. The proof follows the intuition discussed in Section \ref{sec:desirable} that the desirable properties of the learned representation benefit the downstream training. In a nutshell, we derive the bound by writing the error in terms of singular values and vectors of $\pmb{F}$ and then applying Theorem \ref{theo:informal}. }





{In Eq. \eqref{eq: upper_bound_error}, the error is decomposed into bias and variance.}
The bias {captures the mismatch between the average prediction of the model and the ground-truth labels. It depends on the magnitude of the prominent singular values as well as the alignment of the corresponding singular vectors with the ground-truth labels. Contrastive learning reduces the bias by aligning the first $\bar{K}$ singular vectors with ground-truth labels (Theorem \ref{theo:informal}), thus producing a small second term in the bias. The variance quantifies the sensitivity to label noise, and is controlled by the magnitude of the non-prominent singular values, which is guaranteed to be small by Theorem \ref{theo:informal}. The regularization parameter $\beta$ appears in both terms and can be tuned as a trade-off between underfitting and overfitting. 

With small enough $\delta$ and $\xi$, one can select a small $\beta$ to not explicitly penalize the variance much. This results in a small bias, and subsequently a small total error. For example, when there exists a $\beta$ such that $\sqrt{\delta}+\xi \ll \beta \ll 1$, the error $\approx\sigma^2\frac{\bar{K}}{n}(\frac{1}{\beta+1})^2$, which is the inevitable cost of achieving a small bias, when the representation matrix $\pmb{F}$ has $\bar{K}$ prominent singular values. 
} 

\begin{figure*}[t]
    \centering
    \subfloat[\label{subfig: singular_value}]{
	\includegraphics[width=.24\textwidth]{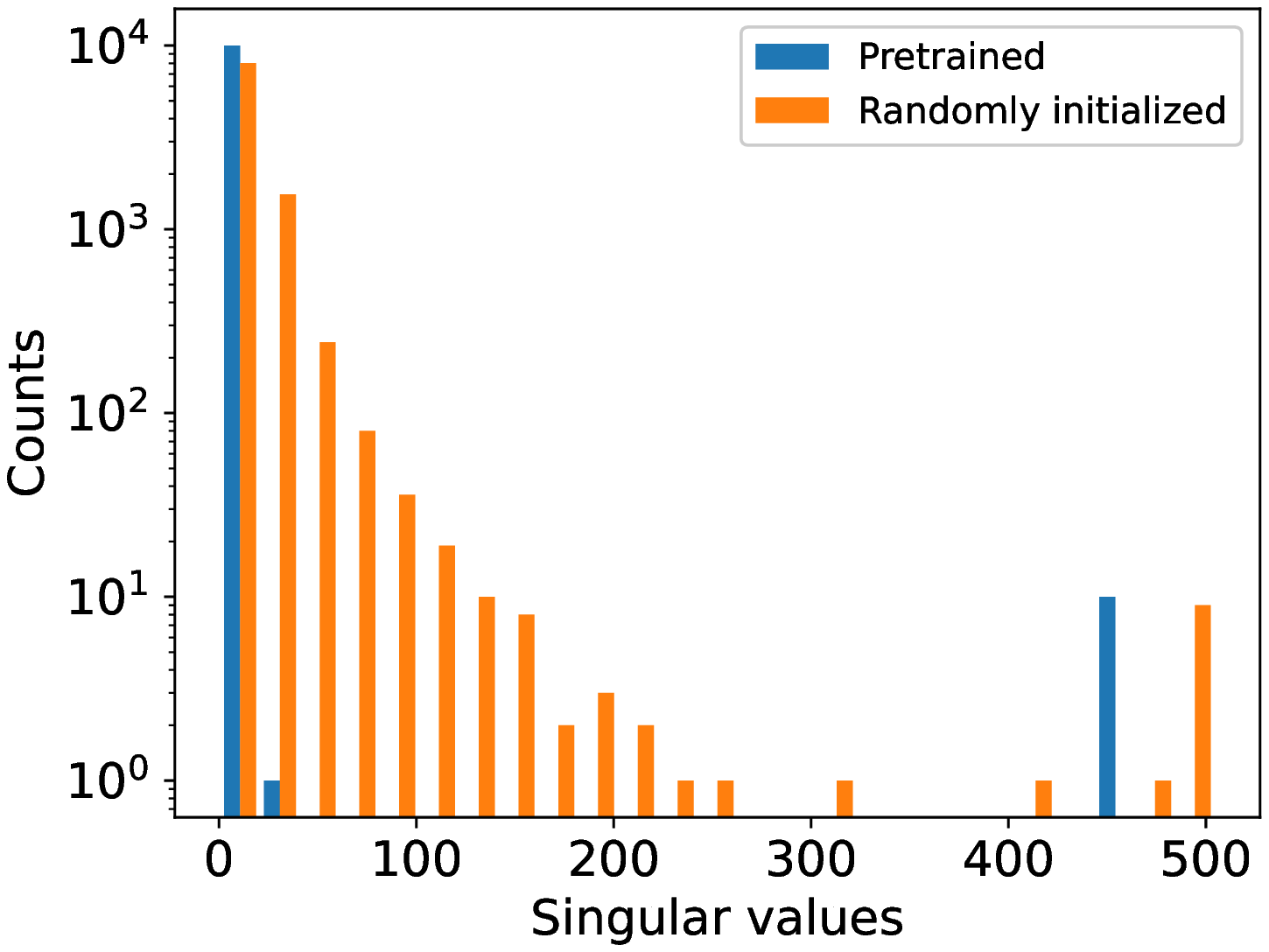}
	}
	\subfloat[\label{subfig: alignment}]{
	\includegraphics[width=.24\textwidth]{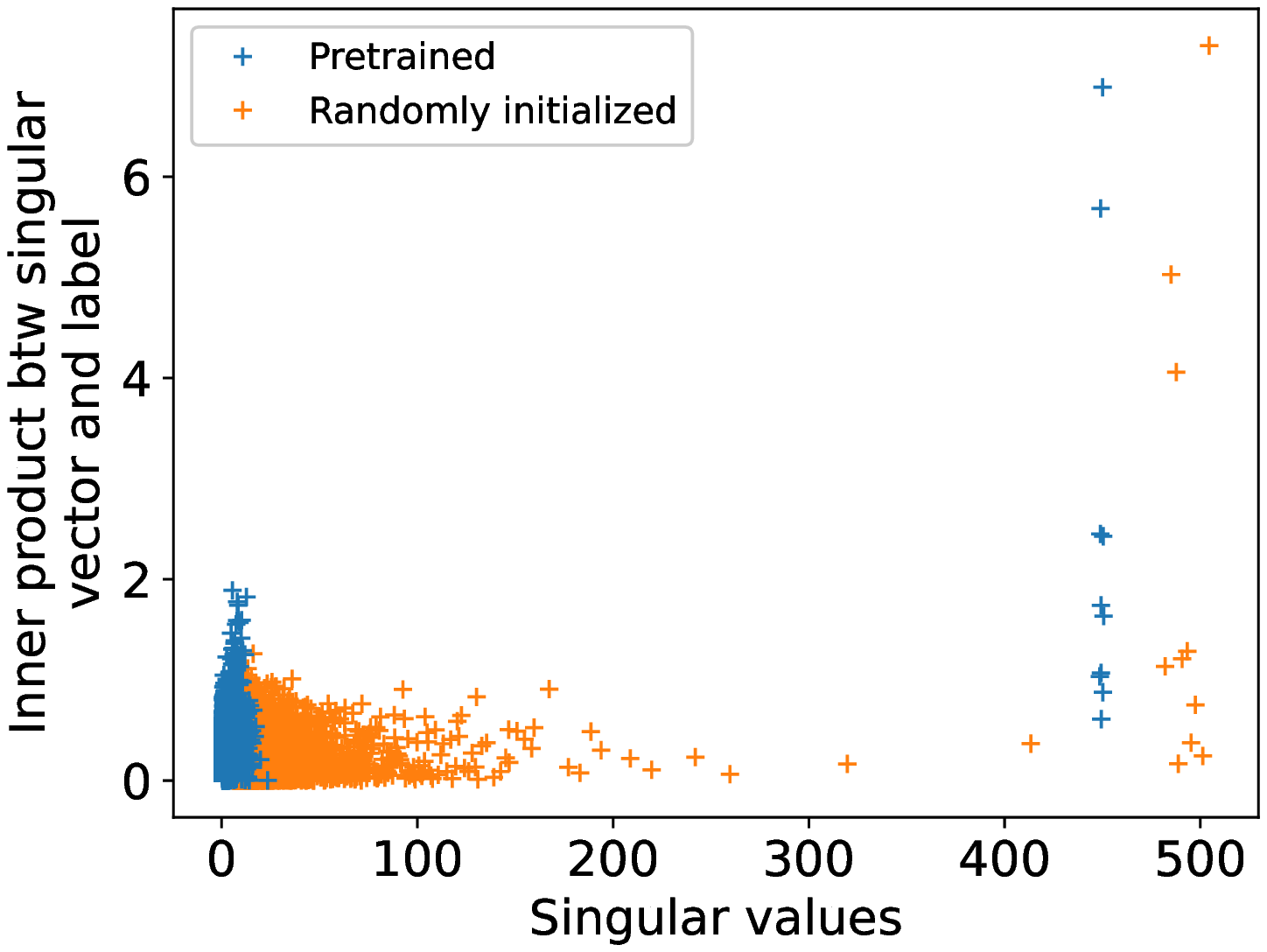}
	}
    \subfloat[\label{subfig: training_curve}]{
	\includegraphics[width=.235\textwidth]{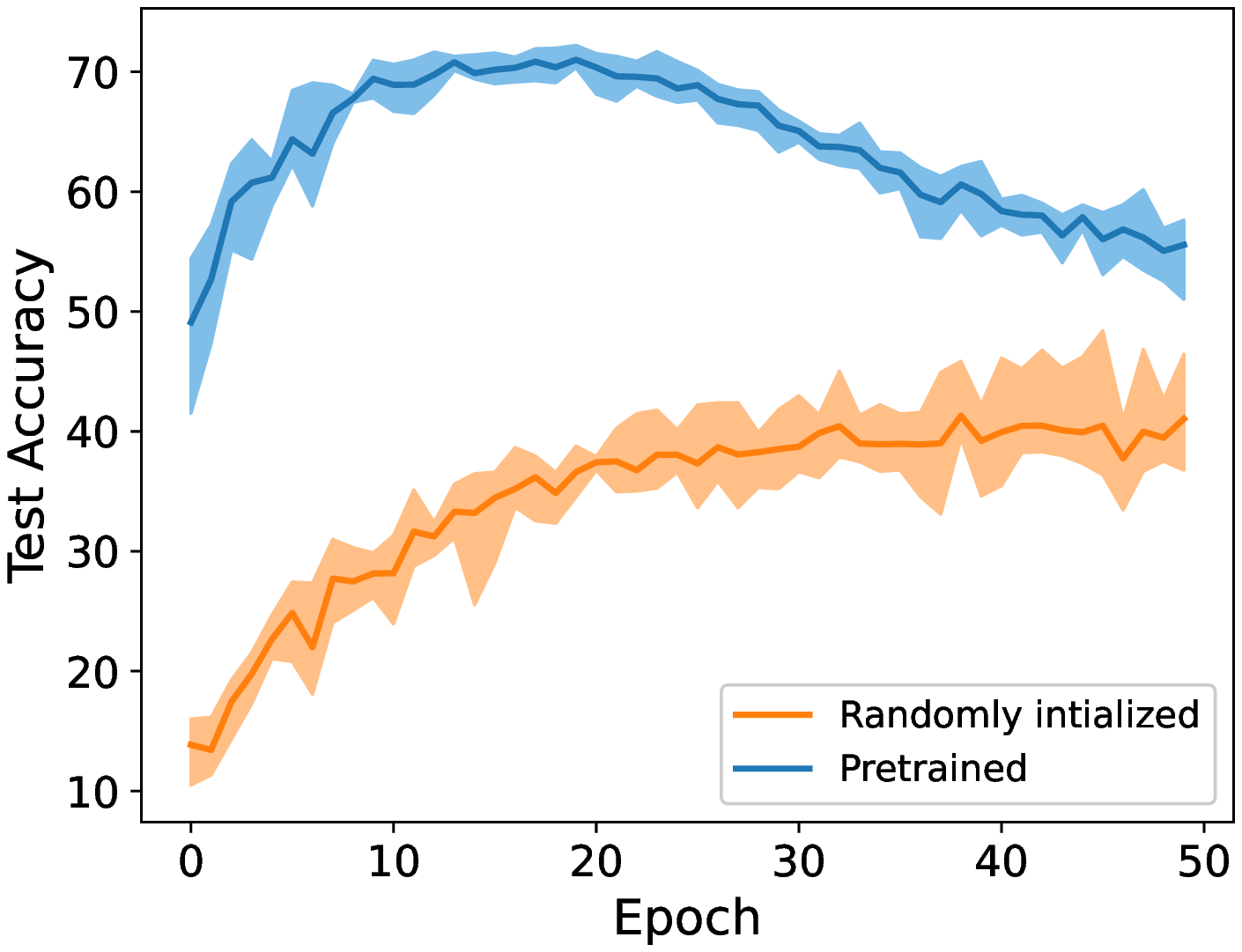}
	}
	\subfloat[\label{subfig: loss_noisy}]{
	\includegraphics[width=.242\textwidth]{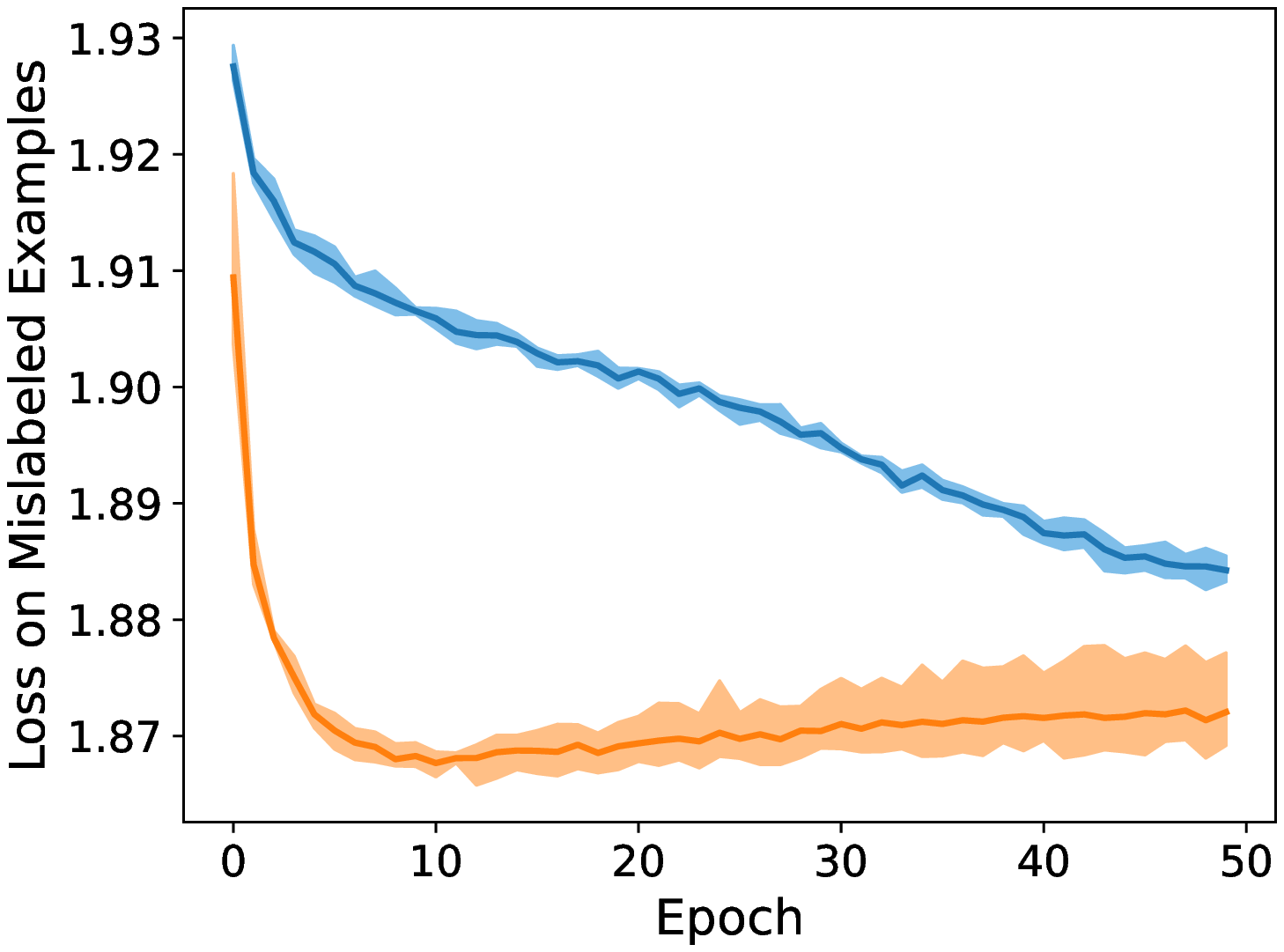}
	}
    \caption{Jacobian spectrum and dynamics of training a randomly initialized vs. fine-tuning a pre-trained ResNet32 on CIFAR-10 with 80\% randomly flipped labels.
   (a) distribution of singular values of the initial Jacobian, (b) alignment of the clean labels with the initial Jacobian, (c) test accuracy and (d) loss on mislabeled data points within the first 50 epochs.
    While pre-training does not improve the alignment of the Jacobian with the labels, it significantly slows down overfitting at the beginning by shrinking the smaller singular values of the Jacobian matrix.
    } 
    \label{fig:singular_value_cifar10}
\end{figure*}

\subsubsection{Random Label Flipping}
Next, we study the case where the label noise $\Delta \pmb{Y}\!=\!\hat{\pmb{Y}}\!-\!\pmb{Y}$ is generated by flipping a fraction of the clean labels at random. 
Formally, for an example $\pmb{x}_i$ belongs to class $j$ with $\pmb{y}_i=\pmb{e}_j$, if its label is flipped to class $k$, we have $\Delta \pmb{y}_i=\pmb{e}_k-\pmb{e}_j$.
We introduce the following notations to analyze the case of asymmetric label noise which mimics the real-world noise, where wrong labels are generated in a (sub)class-dependent way.
Let $m_{\bar{k}}$ be the number of mislabeled examples in subclass $\bar{k}$, $m_{\bar{k},k}$ be the number of mislabeled examples in subclass $\bar{k}$ that are labeled as class $k$, and $Z_{k}$ be the set of sub-classes in class $k$. {We define $c_{k|\bar{k}}\coloneqq\frac{m_{\bar{k},k}}{m_{\bar{k}}}$ for all $\bar{k}$,$k$ such that $\bar{k}\notin Z_{k}$, to be the fraction of mislabeled examples in sub-class $\bar{k}$ that are mislabeled as $k$. We have that $\sum_{k:~\bar{k}\notin Z_{k}}c_{k|\bar{k}}=1,~ \forall \bar{k}\in[\bar{K}]$.  When $c_{k|\bar{k}}=\frac{1}{K-1}~~ \forall k\in[K], \bar{k}\in[\bar{K}]$, the noise is symmetric.
}

{For simplicity we assume $\xi=0$. The general case of $\xi\geq0$ requires more involved analysis which we discuss in the Appendix.}
The following theorem shows that for a dataset with compact and distinguishable sub-class structure 
the linear classifier trained on the representations obtained by contrastive learning can 
{recover the clean label for all training data}.
\begin{theorem}[Asymmetric Noise]\label{theorem: flipping_accuracy}
For a dataset 
with ${K}$ classes and $\bar{K}$
compact and distinguishable sub-class structure (\textit{c.f.} assumptions \ref{assump: block_with_off_diag} \ref{assump: ratio}) with $\xi=0$, 
let $n_{\min}, n_{\max}$ be the size of the smallest and largest sub-class, and $\alpha$ be the 
fraction of {mislabeled examples} in the training set. {Let $c_{\max}\coloneqq\max_{k\in[K],\bar{k}\in[\bar{K}]} c_{k|\bar{k}}\in[\frac{1}{K-1}, 1]$ 
be the maximum fraction of wrong labels in a subclass that are flipped to another class.
}
Then as long as 
{
\begin{align}
{
    \label{eq: tolerance}
    \alpha < \frac{1}{1+\frac{n_{\max}}{n_{\min}}c_{\max}}- \mathcal{O}\left(\frac{\sqrt{\delta}}{\beta}\right), }
\end{align}
}

 a linear model trained by minimizing the objective in Eq. \eqref{eq:objective} with the representations
obtained by minimizing contrastive loss in Eq. \eqref{eq:contrastive} can predict the ground-truth labels for all training examples, i.e., 
\begin{align}
\nonumber
   \forall i,~\argmax_{j\in[K]} (\pmb{F}\hat{\pmb{W}_{}}
    )_{i,j} \!\!=\! \arg_{j\in [K]}({\pmb{Y}}_{i,j}=1) \!\!.
\end{align}

\end{theorem}
In other words contrastive learning can prevent the linear model from memorizing any wrong label even under large noise. Theorem \ref{theorem: flipping_accuracy} also shows that the model can tolerate more noise when the sub-class structure is more compact, i.e., $\delta$ is smaller, or the noise is more symmetric, or the sub-classes are more balanced. The following corollary for symmetric noise is simply obtained by setting $c_{\max}=\frac{1}{K-1}$ in Theorem \ref{theorem: flipping_accuracy}.
{
\begin{corollary}[Symmetric Noise]\label{corollary:symmetric}
For symmetric noise, under the same assumption as in Theorem \ref{theorem: flipping_accuracy}, as long as
\begin{align}
{
    \label{eq: tolerance}
    \alpha < \frac{K\!-\!1}{K+\frac{n_{\max}}{n_{\min}}-1}- \mathcal{O}\left(\sqrt{\delta}\right),}
\end{align}
 a linear model trained by minimizing the objective in Eq. \eqref{eq:objective} with the representations
obtained by minimizing contrastive loss in Eq. \eqref{eq:contrastive} can predict the ground-truth labels for all training examples.
\end{corollary}
}
If we further let $\delta=0$ and $n_{\max}/n_{\min}=1$,  
we get $(K-1)/K$ noise tolerance. We note that this, however, does not imply that a dataset with more classes necessarily has a higher noise tolerance. In Appendix \ref{sec: perturb_eigen}, we show that  less distinguishable sub-class structure, i.e. $\xi>0$, introduces a $\mathcal{O}(\bar{K}^{5/2}\xi)$ perturbation to the singular values and a $\mathcal{O}(\bar{K}^{5/2}\xi)$ rotation in the direction of singular vectors of the representation matrix. Datasets with more classes usually contains more sub-classes, which 
greatly reduces the noise tolerance. {This is also reflected by our empirical results (Figure \ref{fig:train_acc_gt}) where the performance of the linear model is worse on CIFAR-100 compared to CIFAR-10 under the same noise level.}
\subsection{Contrastive Learning Slows down Overfitting for Fine-tuning}
\label{sec: training_the_whole}


In the previous section, we showed that training a linear model on representations learned by contrastive learning is provably robust.
Here, we study fine-tuning all layers of the deep network. Interestingly, as is shown in Fig. \ref{subfig: training_curve}, finetuning achieves a very high test accuracy under 80\% label noise in the early phase of training.

Recall that the theoretical guarantee for linear model (theorems \ref{theorem: gaussian_loss} and \ref{theorem: flipping_accuracy}) is obtained by examining singular values and singular vectors of $\pmb{F}$. Here, we use a similar idea to understand benefits of contrastive learning on robustness when all the layers are trained.
Intuitively, 
during the early stage of training, it is natural to assume that the gradient does not considerably change, and therefore the model is nearly linear. In this case, the initial Jacobian matrix plays the same role as the representation matrix, $\pmb{F}$, to the linear model. {This is supported by the recent studies suggesting the following properties of training neural networks: the early learning dynamics can by mimicked by training a linear model \cite{hu2020surprising}, SGD on neural networks learns a linear model first \cite{kalimeris2019sgd},
and a network that provides a better 
alignment between prominent directions of the Jacobian matrix and the label vector is more likely to generalize well \cite{oymak2019generalization}. }

We examine the SVD of the Jacobian of a ResNet pretrained with contrastive learning and compare it to that of a randomly initialized network. Fig. \ref{subfig: singular_value}, \ref{subfig: alignment} present the distribution of singular values and the alignment of singular vectors with clean labels. The Jacobian is computed on a random sample of 1000 data points from CIFAR10. 
Interestingly, Fig. \ref{subfig: alignment} shows that while pre-training does not considerably improve (in Appendix \ref{apx:alignment} we show a slight improvement)
the alignment between singular vectors of the Jacobian and the clean label vector, it greatly shrinks the smaller singular values of the Jacobian, as is illustrated by Fig. \ref{subfig: singular_value}. As a result, it takes substantially longer for the pre-trained network to overfit the noisy labels.  
As Fig. \ref{subfig: loss_noisy} shows, while a randomly initialized network experience a sharp drop in loss of noisy labeled data points during the first few epochs of training, it takes much longer for a pre-trained network to overfit the noise. Details of the experiment can be found in Appendix \ref{apx:frz_unfrz}. \looseness=-1

\subsection{Contrastive Learning Boosts Robust Methods}
As discussed, pre-training the network with contrastive learning effectively shrinks the smaller singular values of the Jacobian and slows down overfitting the noisy labels.
The initial level of robustness provided by contrastive learning can be leveraged by 
existing robust training methods to achieve a superior performance under extreme noise levels. Next, we briefly discuss three methods that prevent the pre-trained network from overfitting the noisy labels,
through  regularization \cite{liu2020early,zhang2017mixup}, or identifying clean examples \cite{mirzasoleiman2020coresets}. 

\vspace{-2mm}
\paragraph{ELR} \cite{liu2020early} 
regularizes the loss by 
$\frac{1}{n}\sum_{i=1}^n \log(1-\pmb{p}(x_i)^{\top} \pmb{t}(x_i))$ to 
encourage the alignment between the model prediction $\pmb{p}(x)$ and the 
running average of the model outputs in previous rounds $\pmb{t}(x)$. The effectiveness of ELR is attributed to the early-learning phenomenon where the model first fits the correct labels 
and then memorizes the noisy ones 
\cite{oymak2019generalization}. 
Effectively, the regularization term stretches the prediction toward the clean labels predicted early by the model.
However, under extreme label noise, the memorization phase starts very early, 
and does not let
the model to learn clean labels 
and high-quality targets. 
As discussed, contrastive learning makes a large separation between learning and memorization and gives the network enough time to learn high-quality targets. As we show in our experiments, applying ELR to fine-tune the network learned by contrastive learning significantly boosts the generalization performance.

\vspace{-2mm}
\paragraph{Mixup} \cite{zhang2017mixup} extends the training distribution by linear interpolations of feature vectors and their associated labels: $\hat{\pmb{x}} = \lambda \pmb{x}_i + (1 - \lambda)\pmb{x}_j ,
\hat{\pmb{y}} = \lambda \pmb{y}_i + (1 - \lambda)\pmb{y}_j$, where $\lambda\sim \text{Beta}(\alpha,\alpha)$.
In doing so, mixup makes linear transition in the decision boundary between classes and provide a smoother estimate of uncertainty. Larger $\alpha$ prevents overfitting by generating examples that are less similar to the training examples and are more difficult for the network to memorize.
In our experiments, we show that 
the network learned by contrastive learning can be robustly fine-tuned by mixup to achieve a superior generalization performance. 
\begin{table*}[t]
\centering
\caption{Average test accuracy (3 runs) on CIFAR-10 and CIFAR-100. The best test accuracy is marked in bold. We note the higher performance of methods that use SimCLR (CL) pretraining, especially under higher noise levels.  In particular, under $80\%$ noise, methods see an average of $27.18\%$, and $15.58\%$ increase in test accuracy for CIFAR-10 and CIFAR-100 respectively. 
Results marked with $(^*)$ are reproduced from publicly available code. E2E refers to end to end fine-tuning the pre-trained network. 
}
\label{tab:cifar-table}
{\footnotesize
\begin{tabular}{c|c|c|c|c|c|c|c|c}
\toprule
Dataset     & \multicolumn{4}{c|}{CIFAR-10}        & \multicolumn{4}{c}{CIFAR-100}       \\ \midrule
Noise Type  & \multicolumn{3}{c|}{Sym} & Asym      & \multicolumn{3}{c|}{Sym} & Asym      \\ \midrule
Noise Ratio & $20$          & $50$   & $80$      & $40$        & $20$          & $50$        & $80$ & $40$        \\ \midrule
F-correction    & $85.1 \pm 0.4$ & $76.0 \pm 0.2$ & $34.8\pm4.5$ & $83.6 \pm 2.2$ & $55.8 \pm 0.5$ & $43.3 \pm 0.7$ & $-$ & $42.3 \pm 0.7$ \\
Decoupling    & $86.7 \pm 0.3$ & $79.3 \pm 0.6$ & $36.9 \pm 4.6$ & $75.3 \pm 0.8$ & $57.6 \pm 0.5$ & $45.7 \pm 0.4$  & $-$ & $43.1 \pm 0.4$ \\
Co-teaching    & $89.1 \pm 0.3$ & $82.1 \pm 0.6$  & $16.2 \pm 3.2$& $84.6 \pm 2.8$ & $64.0 \pm 0.3$  & $52.3 \pm 0.4$ & $-$ & $47.7 \pm 1.2$ \\
MentorNet   & $88.4 \pm 0.5$ & $77.1 \pm 0.4$ & $28.9 \pm 2.3$ &$77.3 \pm 0.8$ & $63.0 \pm 0.4$ & $46.4 \pm 0.4$ & $-$ & $42.4 \pm 0.5$ \\
D2L   & $86.1 \pm 0.4$ & $67.4 \pm 3.6$ & $10.0 \pm 0.1 $& $85.6 \pm 1.2$ & $12.5 \pm 4.2$ & $5.6 \pm 5.4$ & $-$ & $14.1 \pm 5.8$ \\
INCV   & $89.7 \pm 0.2$ & $84.8 \pm 0.3$ & $52.3 \pm 3.5$ & $86.0 \pm 0.5$ &  $60.2 \pm 0.2$        &  $53.1 \pm 0.4$  & $-$ & $50.7 \pm 0.2$ \\
T-Revision   & $79.3 \pm 0.5$ & $78.5 \pm 0.6$ & $36.2 \pm 1.6$& $76.3 \pm 0.8$ &  $52.4 \pm 0.3$        &  $37.6 \pm 0.3$  & $-$ & $32.3 \pm 0.4$ \\
L\_DMI   & $84.3 \pm 0.4$ & $78.8 \pm 0.5$ & $20.9 \pm 2.2$ & $84.8 \pm 0.7$ &  $56.8 \pm 0.4$        &  $42.2 \pm 0.5$  & $-$ & $39.5 \pm 0.4$ \\

Crust$^*$ & $ {{85.3} \pm {0.5}}$  & $ {{86.8} \pm {0.3}}$ & ${33.8} \pm {1.3}$ & ${{76.7} \pm {3.4}}$ & ${{62.9} \pm {0.3}}$ & $ {{55.5} \pm {1.1}}$ & $ {{18.5} \pm {0.8}}$ & ${{52.5} \pm {0.4}}$  \\

Mixup & $ {{89.7} \pm {0.7}}$  & $ {{84.5} \pm {0.3}}$ & ${40.7} \pm {1.1}$ &$ { {86.3} \pm {0.1}}$ & ${ {64.0} \pm {0.4}}$ & $ {{53.4} \pm {0.5}}$ & $ {{15.1} \pm {0.1}}$ & $ { {54.4} \pm {2.0}}$  \\

ELR$^*$ & $ {{90.6} \pm {0.6}}$  & $ {{87.7} \pm {1.0}}$ & ${69.5} \pm {5.0}$ & ${{86.6} \pm {2.9}}$ & ${{63.6} \pm {1.7}}$ & $ {{52.5} \pm {4.2}}$ & $ { {23.4} \pm {1.9}}$ &  $ {{59.7} \pm {0.1}}$  \\\hline

CL+E2E$^*$ & $ {{88.8} \pm {0.5}}$  & $ { {82.8} \pm {0.2}}$ & ${72.0} \pm {0.3}$ & ${ {83.5} \pm {0.5}}$ & ${{63.5} \pm {0.2}}$ & $ {{56.1} \pm {0.3}}$ & $ \bf{{{36.7} \pm {0.3}}}$ &  $ {{52.4} \pm {0.2}}$  \\

CL+Crust$^*$ & $ {{86.5} \pm {0.7}}$  & $ { {87.6} \pm {0.3}}$ & $\bf{{77.9} \pm {0.3}}$ &$ { {85.9} \pm {0.4}}$ & ${{63.0} \pm {0.8}}$ & $\bf{{{58.3} \pm {0.1}}}$ & $ { {34.8} \pm {1.5}}$ & $ { {53.3} \pm {0.7}}$\\

CL+Mixup$^*$ & $ {{90.8} \pm {0.2}}$  & $ { {84.6} \pm {0.4}}$ & ${74.8} \pm {0.3}$ &$ { {87.5} \pm {1.3}}$ & ${ {64.4} \pm {0.4}}$ & $ {{55.5} \pm {0.1}}$ & $ { {30.3} \pm {0.4}}$ & $ { {55.5} \pm {0.8}}$  \\

CL+ELR$^*$ & $\bf{{{91.3} \pm {0.0}}}$  & $ {\bf{{89.1} \pm {0.1}}}$ & ${77.7} \pm {0.2}$ &$ {\bf{{89.7} \pm {0.3}}}$ & ${\bf{{64.7} \pm {0.2}}}$ & ${{55.6} \pm {0.2}}$ & $ { {35.9} \pm {0.3}}$ & $ \bf{{{63.6} \pm {0.1}}}$ \\\bottomrule
\end{tabular}
}
\end{table*}

\vspace{-5mm}
\paragraph{CRUST} \cite{mirzasoleiman2020coresets} provides provable robustness guarantees by extracting clean examples that cluster closely in the gradient space
based on the following observation: 
as the nuisance space is very high dimensional, data points with noisy labels spread out in the gradient space. In contrast, the information space is low-dimensional and data points with clean labels that have similar gradients cluster closely together. 
Central clean examples in the gradient space can be efficiently extracted by maximizing a submodular function. 
To enable Crust to find the clean examples under extreme noise, 
we first fine-tune the entire network on noisy labels for around 20 epochs and then randomly label half of the examples with the prediction of the model. 
As discussed, pre-training the network with contrastive learning shrinks the smaller singular values of the Jacobian. 
This allows the clean examples to make clear clusters around the large singular directions and be easily extracted.
In our experiments, we show that the pre-trained network can significantly boost Crust's performance under extreme noise.

%% file: experiments.tex
\section{Experiments}\label{sec: experiments}
We evaluate the effectiveness of contrastive learning in boosting the robustness of deep networks under various levels of label noise.
We first consider fine-tuning all layers of a network pre-trained with contrastive learning on noisy labels, and show that it can achieve a comparable generalization performance to the state-of-the-art robust methods.
Then, we show that the structure of the representation matrix obtained by contrastive learning can be leveraged by robust methods to achieve a superior generalization performance under extreme noise levels.

For our evaluation, we use artificially corrupted versions of CIFAR-10 and CIFAR-100 \cite{krizhevsky2009learning}, as well as a subset of the real-world dataset Webvision \cite{li2017webvision}, which naturally contains noisy labels. Our method was developed using PyTorch \cite{paszke2017automatic}. We use 1 Nvidia A40 for all experiments.

\textbf{Baselines}.  
We compare our results with
many commonly used baselines for robust training against label noise:  (1) F-correction \cite{patrini2017making} is a two step process, where a neural network is first trained on noisily-labelled data, then retrained using a corrected loss function based on an estimation of the noise transition matrix. (2) Decoupling \cite{malach2017decoupling} is a meta-algorithm that trains two networks concurrently, only training on examples where the two networks disagree. (3) Co-teaching \cite{han2018co} also trains two networks simultaneously. Each network selects subsets of clean data with high probability for the other network to train on. (4) MentorNet \cite{jiang2018mentornet} uses two neural networks, a student and a mentor.  The mentor dynamically creates a curriculum based on the student, while the student trains on the curriculum provided by the mentor. (5) D2L \cite{ma2018dimensionality} learns the training data distribution, then dynamically adapts the loss function based on the changes in dimensionality of subspaces during training. (6) INCV \cite{chen2019understanding} identifies random subsets of the training data with fewer noisy labels, then applies Co-teaching to iteratively train on subsets found with the most clean labels. (7) T-Revision \cite{xia2019anchor} learns the transition matrix efficiently using an algorithm that does not rely on known points with clean labels. (8) L\_DMI \cite{xu2019l_dmi} uses a novel information-theoretic loss function based on determinant based mutual information. (9) ELR \cite{liu2020early} uses semi-supervised learning techniques to regularize based on the early-learning phase of training, to ensure the noisy labels are not overfit.  (10) CRUST \cite{mirzasoleiman2020coresets} dynamically selects subsets of clean data points by clustering in the gradient space. (11) Mixup \cite{zhang2017mixup} smooths the decision boundary by adding linear interpolations of feature vectors and their labels to the dataset.

\subsection{Empirical Results on Artificially Corrupted CIFAR}\label{sec:exp_cifar}
We first evaluate our method on CIFAR-10 and CIFAR-100, which each contain $50,000$ training images, and $10,000$ test images of size $32 \times32 \times 3$.  CIFAR-10 and CIFAR-100 have $10$ and $100$ classes respectively.  We use the same testing protocol as \cite{xu2019l_dmi,liu2020early,xia2019anchor}, by evaluating our method on symmetric and asymmetric label noise.  For both CIFAR-10 and CIFAR-100, we use symmetric noise ratios of $0.2$, $0.5$, $0.8$, and an asymmetric noise ratio of $0.4$.

In our experiments, we first pre-train ResNet-32 \cite{he2016deep} using SimCLR \cite{paper_simclr,simclr}  for 1000 epochs using the Adam optimizer \cite{kingma2014adam} with a learning rate of $3 \times 10^{-4}$, a weight decay of $1 \times 10 ^{-6}$ and a batch size of 128.  When pre-training, the last linear layer of ResNet-32 is replaced with a 2-layer projection head with an output dimensionality of 64. When pretraining is finished, we replace the projection head with a new, randomly initialized classification layer, and begin training normally.  
We also report the results when ELR, Mixup, and Crust are applied to fine-tune the pre-trained network.
For ELR, we use $\beta = 0.7$ for the temporal ensembling parameter, and $\lambda = 3$ for the regularization strength. For mixup, we use a mixup strength of $\alpha = 1$.  For Crust, we choose a coreset ratio of $0.5$.

The results are shown in Table \ref{tab:cifar-table}.  We note that SimCLR pretraining leads to an across the board improvement for Crust, ELR, and Mixup.  For lower noise ratios, the improvement is marginal.  However, for extreme noise ratios, the improvement is more dramatic.  
In particular, pre-training boosts the performance of Crust by up to 44.1\%, ELR by up to 8.2\%, and Mixup by up to 34.1\% under 80\% noise.
We also note that under $80\%$ noise, SimCLR pretraining alone outperforms all methods without SimCLR pretraining on CIFAR-10 and CIFAR-100.

{
\textbf{Effects of Network Size}
Empirically, larger networks 
trained for longer can achieve smaller contrastive loss \cite{paper_simclr}, thus providing a representation closer to the optimal. In this regard, increasing network size should further improve the robustness to label noise. We confirm this by comparing pretrained ResNet34 (with 46x params) and ResNet32 in Table \ref{tab:size}. Both models are fintuned with Mixup. We see that increasing the network size can greatly improve the performance.  \looseness=-1}
\begin{table}[t]
\caption{Comparison between ResNet34 and ResNet32. We finetune the contrastively pretrained network with Mixup on CIFAR10/100 with different fractions of noise.
}\vspace{1mm}
\centering
\label{tab:size}
{
\begin{small}
\begin{tabular}{c|c|c|c|c}
\toprule
Dataset & \multicolumn{2}{c|}{CIFAR-10}   & \multicolumn{2}{c}{CIFAR-100}  \\ 
\hline
Noise & Sym 80 & Asym 40 & Sym 80 & Asym 40\\
\hline
ResNet-32 & 74.8 $\pm$.3 & 87.5$\pm$1.3 & 30.3$\pm$.4 & 55.5$\pm$.8 \\
ResNet-34 & 90.8$\pm$.6 & 90.4$\pm$.4 & 69.4$\pm$.3& 65.2$\pm$.2 \\\bottomrule
\end{tabular}
\end{small}
}
\vspace{-.7cm}
\end{table}



\begin{table}[t]
	\caption{Test accuracy on mini WebVision. The best test accuracy is marked in bold.  SimCLR (CL) pre-training leads to average improvements of $4.11\%$ and $3.20\%$ for mini Webvision and ImageNet respectively. 
	}
	\label{tab:webvision-table}
	\centering
	{\small
	\begin{tabular}{c|cc|cc}
		\toprule
		& \multicolumn{2}{c|}{WebVision} & \multicolumn{2}{c}{ImageNet} \\
		Method & Top-1 & Top-5 & Top-1 & Top-5 \\
		\midrule
    F-correction & 61.12 & 82.68 & 57.36 & 82.36 \\
    Decoupling & 62.54 & 84.74 & 58.26 & 82.26 \\
    Co-teaching & 63.58 & 85.20 & 61.48 & 84.70 \\
    MentorNet & 63.00 & 81.40& 57.80 & 79.92 \\
    D2L & 62.68 & 84.00 & 57.80 & 81.36\\
    INCV & 65.24 & 85.34 & 61.60 & 84.98 \\
    {Crust} & {72.40} & {89.56} & {67.36} & {87.84}\\
    Mixup & 71.38 & 87.36 & 68.34 & 88.44\\
    ELR & 76.26 & 91.26 & 68.71 & 87.84 \\\hline
    CL + E2E & 71.84 & 88.84 & 68.48 & 89.32\\
    CL + Mixup & 76.34 & 90.52 & \bf{72.25} & \bf{89.72}\\
    CL + ELR & \bf{79.52} & \bf{93.80} & 71.20 & 90.80\\
		\bottomrule
	\end{tabular}}
\end{table}

\subsection{Empirical Results on WebVision}
WebVision is large scale image dataset with noisy labels~\cite{li2017webvision}. It contains 2.4 million images crawled from Google Images search and Flickr that share the same 1000 classes as the ImageNet dataset. 
The noise ratio in classes varies from 0.5\% to 88\%, and the number of images per class varies from 300 to more than 10,000 (Fig. 4 in \cite{li2017webvision} shows the noise distribution).
We follow the setting in~\cite{jiang2018mentornet} and create a mini WebVision dataset that consists of the top 50 classes in the Google subset with 66,000 images. We use both WebVision and ImageNet test sets for testing the performance of the model. We train InceptionResNet-v2~\cite{szegedy2017inception} for 120 epochs with a starting learning rate of $0.02$, which we anneal by a factor of $0.01$ at epochs 40 and 80.  We use the SGD optimizer with a weight decay of $1 \times 10^{-3}$, and a minibatch size of 32.  For Mixup and ELR, we use the same hyperparameters as CIFAR.

Table~\ref{tab:webvision-table} shows the Top-1 and Top-5 accuracy of different methods evaluated on WebVision and ImageNet.  We see that for both ELR and Mixup, SimCLR pretraining leads to average improvements of $4.11\%$ and $3.20\%$ for mini Webvision and ImageNet respectively.  Furthermore, we note that SimCLR pretraining on its own outperforms every method without SimCLR pretraining, except ELR and Crust. 


%% file: conclusion.tex
\section{Conclusion}
We showed that representations learned by contrastive learning provably boosts robustness against noisy labels. In particular, contrastive learning  provides a representation matrix that has: 
{(i) a significant gap between {the prominent singular values and the remaining ones;}
(ii) a large alignment between the prominent singular vectors and the clean labels. {The above properties allow a linear layer trained on the representations to effectively learn the clean labels well while barely overfitting the noise.}}
Then we explained why fine-tuning all layers of a network pre-trained with contrastive learning can also achieve a good performance in the early phase. Crucially, contrastive learning greatly reduces the magnitude of nonprominant singular values of the initial Jacobian matrix, which slows down the overfitting.
Finally, we demonstrated that the initial robustness provided by contrastive learning
{can boost robust methods and achieve state-of-the-art performance under extreme noise levels.}
Our results confirm benefits of contrastive pretraining for robust machine learning. \looseness=-1

%% file: appendix.tex
\section{Analysis for Disconnected Subclasses}
In this section we consider the case where $\xi=0$ in assumption \ref{assump: block_with_off_diag}, which implies that the probability of two augmentation data from different subclasses being generated from the same random natural datum is exactly zero. And in section \ref{sec: non_zero_off_diag} we extend the results to any $\xi\in[0, 1)$ via eigenvalue and eigenvector perturbation. We use $\|\cdot\|_1$, $\|\cdot\|_2$ and $\|\|_F$ to denote the $1$-norm, operator norm and Frobenius norm, respectively.

\subsection{Spectral Decomposition of Augmentation Graph} 
An important technical idea we use to formalize the representations obtained by contrastive learning is augmentation graph \cite{paper_ssl_contra_loss}, which is an undirected graph with all augmentation data $\{\pmb{x}_1, \pmb{x}_2, \dots, \pmb{x}_n\}$ as its vertices and $w_{\pmb{x}_i \pmb{x}_j}$ as the weight for edge $(\pmb{x}_i, \pmb{x}_j)$. Let $\pmb{A}$ denote the adjacency matrix of the augmentation graph, that is, each entry $a_{ij}=w_{\pmb{x}_i\pmb{x}_j}$. And the normalized adjacency matrix $\overline{\pmb{A}}$ is defined as
\begin{align}
    \nonumber
    \overline{\pmb{A}} \coloneqq \pmb{D}^{-1/2} \pmb{A} \pmb{D}^{-1/2},
\end{align}
where $\pmb{D} = \text{diag}(w_{\pmb{x}_1}, w_{\pmb{x}_2}, \dots, w_{\pmb{x}_n})$ with $w_{\pmb{x}_i} = \sum_{j=1}^n w_{\pmb{x}_i\pmb{x}_j}$. For simplicity we index the augmentation data in the following way: the first $n_1$ data are from subclass $1$, the next $n_2$ data are from subclass $2$, \dots, the last $n_{\bar{K}}$ data are from subclass $\bar{K}$. Lemma \ref{lemma: A_bar} states an important property of $\overline{\pmb{A}}$.
\begin{lemma}\label{lemma: A_bar}
Assumption \ref{assump: block_with_off_diag} with $\xi=0$ implies that $\overline{\pmb{A}}$ is a block diagonal matrix $\textbf{diag}(\overline{\pmb{A}}_1, \overline{\pmb{A}}_2, \dots, \overline{\pmb{A}}_{\bar{K}})$ where $\overline{\pmb{A}}_{\bar{k}} \in \mathbb{R}^{n_{\bar{k}}\times n_{\bar{k}}}$. This combined with \ref{assump: ratio} gives us: for each block $\overline{\pmb{A}}_{\bar{k}}$, the ratio between two entries in the same column is bounded as follows
\begin{align}
    \nonumber
    \forall \bar{k}\in[\bar{K}], \max_{j, s, t} \frac{\bar{a}_{\bar{k},s,j}}{\bar{a}_{\bar{k}, t, j}} \leq 1+\delta',
\end{align}
where $\bar{a}_{k, i, j}$ are the entries of $\overline{\pmb{A}}_{\bar{k}}$ and $\delta' = (1+\delta)^{3/2}-1 = \frac{3}{2}\delta+\mathcal{O}(\delta^2)$.
\end{lemma}

Let $f_{\min} = \argmin_{f} \mathfrak{C}(f)$ and $\pmb{F}_{\min} = [f_{\min}(\pmb{x}_1) ~ f_{\min}(\pmb{x}_2) \dots f_{\min}(\pmb{x}_n)]^{\top}$, according the theorem in \cite{paper_ssl_contra_loss}, $\pmb{F}_{\min}$ is also the minimizer of the following matrix factorization problem
\begin{align}
    \label{eq: mf}
    \min_{\pmb{F}\in \mathbb{R}^{n\times p}}  \|\overline{\pmb{A}}-\pmb{F}\pmb{F}^{\top}\|_F^2,
\end{align}
and therefore can be further decomposed as 
\begin{align}
    \label{eq: decomp_F}
    \pmb{F}_{\min} = \pmb{F}^* \pmb{\Sigma} \pmb{R},
\end{align}
by Eckart–Young–Mirsky theorem \cite{Eckart_Young_Mirsky}, where $\pmb{F}^* \in \mathbb{R}^{n\times p} = [\pmb{v}_1, \pmb{v}_2, \dots, \pmb{v}_p] \in \mathbb{R}^{n\times p}$, $\pmb{\Sigma} = \textbf{diag}(\sqrt{\lambda_1}, \sqrt{\lambda_2}, \dots,  \sqrt{\lambda_p})$, $\pmb{R}\in \mathbb{R}^{p\times p}$ is some orthogonal matrix, $\lambda_1, \lambda_2, \dots, \lambda_p$ are the $p$ largest eigenvalues of $\overline{\pmb{A}}$ and $\pmb{v}_1, \pmb{v}_2, \dots, \pmb{v}_p$ are the corresponding unit-norm eigenvectors. W.l.o.g.,  we assume $\lambda_1\geq \lambda_2 \geq \dots \geq \lambda_p$. Our following proofs are based on this decomposition. To avoid cluttered notation we drop the subscript of $\pmb{F}_{\min}$, i.e., we use $\pmb{F}$ for the \emph{learned} representation.

\subsection{Properties of Singular Values/Vectors of the Representation Matrix }\label{sec:property_A}
From the above section we know that the singular values/vectors of $\pmb{F}$ are the first $p$ eigenvalues/vectors of $\overline{\pmb{A}}$. For each block $\overline{\pmb{A}}_{\bar{k}}$, let $\lambda_{\bar{k}, 1}, \lambda_{\bar{k},2}, \dots, \lambda_{\bar{k}, n_{\bar{k}}}$ denote the eigenvalues (in descending order) and $\pmb{v}_{\bar{k},1}, \pmb{v}_{\bar{k},2}, \dots, \pmb{v}_{\bar{k}, n_{\bar{k}}}$ denote the corresponding eigenvectors. The eigenvalues of $\overline{\pmb{A}}$ are the list of the eigenvalues of all blocks. The corresponding eigenvectors are the block vectors $(\vec{0}_1,\vec{0}_2,\dots, \vec{0}_{\bar{k}-1}, \pmb{v}_{\bar{k},i}, \vec{0}_{\bar{k}+1}, \dots, \vec{0}_{\bar{K}})\coloneqq \dot{\pmb{v}}_{\bar{k}, i}$ with each $\vec{0}_j$ being a zero vector of length $n_j$. Since $\overline{\pmb{A}}$ is a normalized adjacency matrix, each block $\overline{\pmb{A}}_{\bar{k}}$ is also normalized. Then the largest eigenvalue for each block is exactly $1$, i.e., $\lambda_{\bar{k},1}=1$. It immediately follow Lemma \ref{lemma:largest_egvalue}.
\begin{lemma}\label{lemma:largest_egvalue}
The $\bar{K}$ largest eigenvalues of $\overline{\pmb{A}}$ are all $1$, i.e., $\lambda_1=\lambda_2=\dots=\lambda_{\bar{K}}=1$.
\end{lemma}
As long as $p\geq \bar{K}$, all $\lambda_{\bar{k}, 1}$ and $\pmb{v}_{\bar{k}, 1}$ appear in  the decomposition of $\pmb{F}$. Let $p_{\bar{k}}\geq 1$ be the number of eigenvalues/eigenvectors of block $\overline{\pmb{A}}_{\bar{k}}$ that also appear in the decomposition of $\pmb{F}$. The following Lemmas and Corollaries states other important properties of eigenvalues/eigenvectors of $\overline{\pmb{A}}$ .\looseness=-1

\begin{lemma}\label{lemma: perron_vec} By assumption \ref{assump: ratio}, the $1$-norm of $\pmb{v}_{\bar{k},1}$ has the following lower bound.
\begin{align}
    \nonumber
    \|\pmb{v}_{\bar{k},1}\|_1^2 \geq \frac{n_{\bar{k}}^2}{1+(n_{\bar{k}}-1)(1+\delta')^2} = n_{\bar{k}}-2(n_{\bar{k}}-1)\delta'-\mathcal{O}(\delta'^2).
\end{align}
\end{lemma}
\begin{proof}
Write $\pmb{v}_{\bar{k},1}$ as $[e_1, e_2, \dots, e_{n_{\bar{k}}}]^T$. By Perron-Frobenius theorem, all the elements here are positive since $\overline{\pmb{A}}_{\bar{k}}$ is a positive matrix. Then the quotient of any two elements in $\pmb{v}_{\bar{k},1}$ can be bounded in terms of the entries of $\overline{\pmb{A}}_{\bar{k}}$ \cite{maximal_eigenvec} and therefore $1+\delta'$ by lemma \ref{lemma: A_bar}:
\begin{align}
    \nonumber
    \max_{i,j} \frac{e_i}{e_j} \leq \max_{j,s,t} \frac{\bar{a}_{\bar{k},s,j}}{\bar{a}_{\bar{k}, t, j}} \leq 1+\delta'.
\end{align}
Let $e_{min} = \min_i e_i$, we have 
\begin{align}
    \nonumber
    1= \|\pmb{v}_{\bar{k},1}\|_2^2 = \sum_{i=1}^{n_{\bar{k}}} e_i^2
    \leq  e_{min}^2 + (n_{\bar{k}}-1)(1+\delta')^2 e_{min}^2,
\end{align}
and
\begin{align}
    \nonumber
    \|\pmb{v}_{\bar{k},1}\|_1^2 &= (\sum_{i=1}^{n_{\bar{k}}} e_i )^2  \geq n_{\bar{k}}^2 e_{min}^2.
\end{align}
Combining the preceding two equations yields
\begin{align}
    \nonumber
     \|\pmb{v}_{\bar{k},1}\|_1^2 \geq \frac{n_{\bar{k}}^2}{1+ (n_{\bar{k}}-1)(1+\delta')^2}.
\end{align}
\end{proof}

\begin{lemma}\label{lemma: bounded_squared_eigenv}
The sum of squared eigenvalues of each block $\overline{\pmb{A}}_{\bar{k}}$ can be bounded.
\begin{align}
    \nonumber
    \sum_{i=1}^{n_{\bar{k}}} \lambda_{\bar{k}, i}^2 \leq  \frac{(1+(n_{\bar{k}}-1)(1+\delta')^2)(1+\delta')^2}{n_{\bar{k}}}.
\end{align}
\end{lemma}
\begin{proof}
First we have
\begin{align}
    \label{eq: lambda_c_2}
    \sum_{i=1}^{n_{\bar{k}}} \lambda_{\bar{k}, i}^2=\|\overline{\pmb{A}}_{\bar{k}}\|_F^2 = \sum_{j=1}^{n_{\bar{k}}} \|\pmb{c}_{\bar{k},j}\|_2^2,
\end{align}
where $\pmb{c}_{\bar{k},j} = [\bar{a}_{\bar{k},1,j}, \bar{a}_{\bar{k},2,j}, \dots, \bar{a}_{\bar{k},n_{\bar{k}},j}]^T$ denotes the $j$-th column in $\overline{\pmb{A}}_{\bar{k}}$. Analogous to lemma \ref{lemma: perron_vec}, here we can bound $\|\pmb{c}_{\bar{k},j}\|_2^2$ in terms of $\|\pmb{c}_{\bar{k},j}\|_1^2$ by lemma \ref{lemma: A_bar}
\begin{align}
    \label{eq: c2_c1}
    \|\pmb{c}_{\bar{k},j}\|_2^2 \leq \frac{(1+(n_{\bar{k}}-1)(1+\delta')^2)\|\pmb{c}_{\bar{k},j}\|_1^2}{n_{\bar{k}}^2}.
\end{align}
We also have
\begin{align}
\label{eq: bound_on_c2}
    \|\pmb{c}_{\bar{k},j}\|_1^2 \leq \max_j\|\pmb{c}_{\bar{k},j}\|_1^2 \leq (1+\delta')^2 \min_j\|\pmb{c}_{\bar{k},j}\|_1^2 \leq (1+\delta')^2 \lambda_{\bar{k},1}^2 = (1+\delta')^2,
\end{align}
where the second inequality holds because of assumption \ref{assump: ratio} and the third inequality holds because of Perron-Frobenius theorem. Combining equations \ref{eq: lambda_c_2}, \ref{eq: c2_c1} and \ref{eq: bound_on_c2} completes the proof.
\end{proof}

\begin{corollary}\label{corollary: bound_lambda_2}
The eigenvalues except the $\bar{K}$ largest ones are each upper bounded by 
\begin{align}
    \nonumber
    \lambda_{\bar{k}, i} ~\leq ~ \sqrt{\frac{(1+(n_{\bar{k}}-1)(1+\delta')^2)(1+\delta')^2}{n_{\bar{k}}} - 1}    =  ~\mathcal{O}(\sqrt{\delta'}), \quad \forall i=2,3,\dots, n_{\bar{k}}, \quad \forall \bar{k} \in [\bar{K}].
\end{align}
\end{corollary}
\begin{proof}
By Lemmas \ref{lemma:largest_egvalue} and \ref{lemma: bounded_squared_eigenv} 
\begin{align}
    \nonumber
 \sum_{i=2}^{n_{\bar{k}}} \lambda_{\bar{k},i}^2
    = & \sum_{i=1}^{n_{\bar{k}}} \lambda_{\bar{k},i}^2 - \lambda_{\bar{k},1}^2 \\
    \nonumber
    \leq & \frac{(1+(n_{\bar{k}}-1)(1+\delta')^2)(1+\delta')^2}{n_{\bar{k}}} - 1 \\
    \label{eq:bound_2_to_pk}
    = & 4\delta'+\mathcal{O}(\delta'^2).
\end{align}
Then 
\begin{align}
    \nonumber
    \lambda_{\bar{k}, i} \leq & \sqrt{\sum_{i=2}^{n_{\bar{k}}} \lambda_{\bar{k},i}^2} \\
    \nonumber
    \leq & \sqrt{\frac{(1+(n_{\bar{k}}-1)(1+\delta')^2)(1+\delta')^2}{n_{\bar{k}}} - 1} 
\end{align}
\end{proof}

\begin{corollary}\label{corollary:sum_smallest}
For each block $\overline{\pmb{A}}_{\bar{k}}$, the sum of the eigenvalues from the second to the $p_{\bar{k}}$-th is bounded by
\begin{align}
    \nonumber
    \sum_{i=2}^{p_{\bar{k}}} \lambda_{\bar{k}, i} \leq 2\sqrt{(p_{\bar{k}}-1)}\sqrt{\delta'}+\mathcal{O}(\delta').
\end{align}
And the sum of eigenvalues of $\overline{\pmb{A}}$ from the $\bar{K}+1$-th to the $p$-th is bounded by
\begin{align}
    \nonumber
    \sum_{i=\bar{K}+1}^p \lambda_{i} \leq 2\sqrt{(p-\bar{K})\bar{K}\delta'}+\mathcal{O}(\delta').
\end{align}
\end{corollary}
\begin{proof}
Applying Cauchy–Schwarz inequality to equation \ref{eq:bound_2_to_pk} yields the bound for $\sum_{i=2}^{p_{\bar{k}}} \lambda_{\bar{k}, i}$. Summing both sides of equation \ref{eq:bound_2_to_pk} over $\bar{k}\in[\bar{K}]$ and then applying Cauchy–Schwarz inequality give the bound for $\sum_{i=\bar{K}+1}^p \lambda_{i}$.
\end{proof}

\subsection{Error under Gaussian Noise when $\xi=0$}\label{sec: error_block_adj_matrix}
With the decomposition in Equation \ref{eq: decomp_F}, the learned parameter of the linear model in Equation \ref{eq:closedform_w} can be rewritten as 
\begin{align}
    \nonumber
    \hat{\pmb{W}}^* = \pmb{R}^{\top}\textbf{diag}(\frac{\sqrt{\lambda_1}}{\lambda_1+\beta}, \frac{\sqrt{\lambda_2}}{\lambda_2+\beta}, \dots, \frac{\sqrt{\lambda_p}}{\lambda_p+\beta})\pmb{F}^{*\top}\hat{\pmb{Y}}.
\end{align}
The output on the training set $\pmb{F}\hat{\pmb{W}}^*$ is
\begin{align}
    \nonumber
    \pmb{F}\hat{\pmb{W}}^* = & \pmb{F}^*\pmb{B}\pmb{F}^{*\top}\hat{\pmb{Y}},
\end{align}
where $\pmb{B}=\textbf{diag}(b_1,b_2,\dots,b_p)$ with $b_i=\frac{\lambda_i}{\lambda_i+\beta}$. And the error on training set w.r.t. ground-truth labels can be therefore written in terms of the eigenvalues and eigenvectors of $\overline{\pmb{A}}$
\begin{align}
    \nonumber
    \mathbb{E}_{\Delta \pmb{Y}}\frac{1}{n}\| \pmb{Y} -\pmb{F}\hat{\pmb{W}}^* \|_F^2
    = & \mathbb{E}_{\Delta \pmb{Y}} \left[\frac{1}{n}\|\pmb{Y}-\pmb{F}^*\pmb{B}\pmb{F}^{*\top}(\pmb{Y}+\Delta\pmb{Y})\|_F^2\right] \\
    \nonumber
    = & \frac{1}{n}\|\pmb{Y}-\pmb{F}^*\pmb{B}\pmb{F}^{*\top}\pmb{Y}\|_F^2 + \mathbb{E}_{\Delta \pmb{Y}} \left[ \| \pmb{F}^*\pmb{B}\pmb{F}^{*\top}\Delta \pmb{Y} \|_F^2 \right] \\
    \nonumber
    = & \frac{1}{n}\|\pmb{Y}\|_F^2+ \frac{1}{n}\sum_{i=1}^p\sum_{j=1}^{K}(b_i^2-2b_i)(\pmb{v}_i^{\top}\pmb{y}_j)^2 + \frac{\sigma^2}{n}\sum_{i=1}^p b_i^2\\
    \label{eq: error_eigenvec}
    = & \underbrace{1+ \frac{1}{n}\sum_{i=1}^p\sum_{j=1}^{K}(b_i^2-2b_i)(\pmb{v}_i^{\top}\pmb{y}_j)^2}_{\textbf{bias}^2} + \underbrace{\frac{\sigma^2}{n}\sum_{i=1}^p b_i^2}_{\textbf{variance}},
\end{align}
where $\pmb{y}_j \in \mathbb{R}^n$ is the $j$-th \emph{column} of $\pmb{Y}$.

We first calculate the $\textbf{bias}^2$ term. We have $b_i^2-2b_i \leq 0$ since $b_i\in (0,1]$. Then we drop items with $i\geq \bar{K}+1$ in the summation and apply Lemma \ref{lemma:largest_egvalue} to get an upper bound 
\begin{align}
    \nonumber
    \textbf{bias}^2 \leq & 1 -  \frac{1}{n}\sum_{i=1}^{\bar{K}}\sum_{j=1}^{K}(2b_i-b_i^2)(\pmb{v}_i^{\top}\pmb{y}_j)^2 \\
    \nonumber
    = & 1-\frac{1}{n} \left(\frac{2}{1+\beta} -(\frac{1}{1+\beta})^2\right)\sum_{i=1}^{\bar{K}}\sum_{j=1}^{K}(\pmb{v}_i^{\top}\pmb{y}_j)^2\\
    \nonumber
    = &1-\frac{1}{n} \left(\frac{2}{1+\beta} -(\frac{1}{1+\beta})^2\right)\sum_{\bar{k}=1}^{\bar{K}}\sum_{j=1}^{K}(\dot{\pmb{v}}_{\bar{k},1}^\top\pmb{y}_j)^2
\end{align}
By Perron-Frobenious theorem all elements in $\pmb{v}_{\bar{k},1}$ are positive, thus the sum of elements of $\pmb{v}_{\bar{k},1}$ can be written as $\|\pmb{v}_{\bar{k},1}\|_1$. With the observation that $\dot{\pmb{v}}_{\bar{k},1}^T \pmb{y}_j = \|\pmb{v}_{\bar{k},1}\|_1$ when $\bar{k}$ is a subclass within class $j$ and otherwise $\dot{\pmb{v}}_{\bar{k},1}^T \pmb{y}_j =0$, the above equation can be rewritten as 
\begin{align}
    \label{eq: 2nd_term_in_bias}
    \textbf{bias}^2 \leq & 1-\frac{1}{n} \left(\frac{2}{1+\beta} -(\frac{1}{1+\beta})^2\right)\sum_{\bar{k}=1}^{\bar{K}} \|\pmb{v}_{\bar{k},1}\|_1^2.
\end{align}
Then by Lemma \ref{lemma: perron_vec}, 
\begin{align}
    \nonumber
    \textbf{bias}^2 \leq & 1-\frac{1}{n} \left(\frac{2}{1+\beta} -(\frac{1}{1+\beta})^2\right)\sum_{\bar{k}=1}^{\bar{K}} \left( n_{\bar{k}}-2(n_{\bar{k}}-1)\delta'-\mathcal{O}(\delta'^2) \right)\\
    \nonumber
    = & (\frac{\beta}{1+\beta})^2 + 2 (1-\frac{\bar{K}}{n})\frac{(2\beta+1)}{(\beta+1)^2}\delta'+\mathcal{O}(\delta'^2)\\
    = & (\frac{\beta}{1+\beta})^2 + 3 (1-\frac{\bar{K}}{n})\frac{(2\beta+1)}{(\beta+1)^2}\delta+\mathcal{O}(\delta^2)
\end{align}

Now we bound the $\textbf{variance}$ term. By Lemma \ref{lemma:largest_egvalue}
\begin{align}
\nonumber
    \textbf{variance} =& \frac{\sigma^2}{n}\sum_{i=1}^p b_i^2 \\
    \nonumber
    = & \frac{\sigma^2}{n}\sum_{i=1}^{\bar{K}} b_i^2 + \frac{\sigma^2}{n}\sum_{i=\bar{K}+1}^{p} b_i^2  \\
    \nonumber
    = & \sigma^2\frac{\bar{K}}{n}(\frac{1}{\beta+1})^2+\frac{\sigma^2}{n}\sum_{i=\bar{K}+1}^{p} (1-\frac{\beta}{\lambda_i+\beta}) \\
    \label{eq:tmp_variance}
    = & \sigma^2\frac{\bar{K}}{n}(\frac{1}{\beta+1})^2 + \frac{\sigma^2}{n}(p-\bar{K}) -\frac{\sigma^2}{n}\sum_{i=\bar{K}+1}^{p} \frac{\beta}{\lambda_i+\beta}.
\end{align}
Apply Cauchy–Schwarz inequality and Corollary \ref{corollary:sum_smallest} to bound the summation in the last term
\begin{align}
    \nonumber
    \sum_{i=\bar{K}+1}^{p} \frac{\beta}{\lambda_i+\beta} \geq &~ \frac{\beta(p-\bar{K})^2}{\sum_{i=\bar{K}+1}^p(\lambda_i+\beta)} \\
    \nonumber
    \geq & ~\frac{\beta(p-\bar{K})^2}{2\sqrt{(p_{\bar{k}}-1)}\sqrt{\delta'}+\mathcal{O}(\delta') + (p-\bar{K})\beta} \\
    \nonumber
    = & ~ p-\bar{K} + \frac{2\sqrt{p-\bar{K}}}{\beta}\sqrt{\bar{K}\delta'} + \mathcal{O}(\delta').
\end{align}
Plugging the preceding into Equation \ref{eq:tmp_variance} yields
\begin{align}
\nonumber
    \textbf{variance} \leq & \sigma^2\frac{\bar{K}}{n}(\frac{1}{\beta+1})^2 + \sigma^2 \frac{2\sqrt{\bar{K}(p-\bar{K})}}{n}\frac{\sqrt{\delta'}}{\beta}+\mathcal{O}(\delta')\\
    \nonumber
    = & \sigma^2\frac{\bar{K}}{n}(\frac{1}{\beta+1})^2 + \sigma^2 \frac{\sqrt{6\bar{K}(p-\bar{K})}}{n}\frac{\sqrt{\delta}}{\beta}+\mathcal{O}(\delta)
\end{align}

\subsection{Accuracy under Label Flipping (Proof for Theorem \ref{theorem: flipping_accuracy})}
We study the accuracy by looking at the entries of the output $\pmb{F}\hat{\pmb{W}}^*$.

\begin{align}
    \nonumber
    \pmb{F}\hat{\pmb{W}}^* = & ~ \pmb{F}^*\pmb{B}\pmb{F}^{*\top}(\pmb{Y}+\Delta \pmb{Y}) \\
    \nonumber
    = & \left[ \pmb{F}^*\pmb{B}\pmb{F}^{*\top}(\pmb{y}_1+\Delta \pmb{y}_1), ~ \pmb{F}^*\pmb{B}\pmb{F}^{*\top}(\pmb{y}_2+\Delta \pmb{y}_2), ~ \dots, ~ \pmb{F}^*\pmb{B}\pmb{F}^{*\top}(\pmb{y}_K+\Delta \pmb{y}_K)  \right] \\
    \nonumber
    = & \left[ \sum_{i=1}^p b_i \pmb{v}_i \pmb{v}_i^{\top} (\pmb{y}_1+\Delta \pmb{y}_1), ~ \sum_{i=1}^p b_i \pmb{v}_i \pmb{v}_i^{\top} (\pmb{y}_2+\Delta \pmb{y}_2), ~ \dots, ~ \sum_{i=1}^p b_i \pmb{v}_i \pmb{v}_i^{\top} (\pmb{y}_K+\Delta \pmb{y}_K) \right] \\
    \nonumber
    \coloneqq & \left[ \pmb{z}_1, \pmb{z}_2, \dots, \pmb{z}_K \right]
\end{align}

For convenience we define the notations $C_k$ and $S_{\bar{k}}$ as the sets of indices of examples from class $k$ and subclass $\bar{k}$, respectively
\begin{align}
    \nonumber
    C_k \coloneqq & \{ i: \text{$\pmb{x}_i$ belongs to class $k$} \} \\
    \nonumber
    S_{\bar{k}} \coloneqq & \{ i: \text{$\pmb{x}_i$ belongs to subclass $\bar{k}$} \} = \{ i: \sum_{j=1}^{\bar{k}-1}n_j < i \leq \sum_{j=1}^{\bar{k}} n_j  \}.
\end{align}

Let the notation $\pmb{\mu}^{(j)}$ denote the $j$-th element of vector $\pmb{\mu}$. Then $\pmb{z}_k^{(j)}$ can be written as $ \sum_{i=1}^p b_i \pmb{v}_i^{(j)} \pmb{v}_i^{\top}(\pmb{y}_k+\Delta \pmb{y}_k)$. Let $\bar{k}_j$ denote the subclass that $\pmb{x}_j$ belongs to, i.e., $j\in S_{\bar{k}}$ and define $e_{\min,s} \coloneqq \min_{j} \pmb{v}_{\bar{k}_j,1}^{(s)}$ and $e_{\max, s} \coloneqq \max_{s} \pmb{v}_{\bar{k}_j,1}^{(s)}$. We have the following two lemmas. 
\begin{lemma} \label{lemma: bound_z}
$\pmb{z}_k^{(j)}$ can be bounded
\begin{align}
    \nonumber
    \begin{cases}
    \pmb{z}_k^{(j)} \geq & \frac{1}{\beta+1} e_{\min,j}\dot{\pmb{v}}_{\bar{k}_j,1}^{(j)}n_{\min}(1-\alpha) - \sqrt{n_{\max}}\frac{2\sqrt{p-1}}{\beta}\sqrt{\delta'}- \mathcal{O}(\delta'), \quad j\in C_k  \\
    \pmb{z}_k^{(j)} \leq & \frac{1}{\beta+1}e_{\max,j}\dot{\pmb{v}}_{\bar{k}_j,1}^{(j)}n_{\max}\alpha c_{\max} + \sqrt{n_{\max}c_{\max}}\frac{2\sqrt{p-1}}{\beta}\sqrt{\delta'} + \mathcal{O}(\delta'), \quad j\notin C_k.
    \end{cases}.
\end{align}

\end{lemma}
\begin{proof}
Let $b_{\bar{k}, i}$ denote $\frac{\lambda_{\bar{k}, i}}{\beta+\lambda_{\bar{k}, i}}$. Recalling that one property of the block vector $\dot{\pmb{v}}_{\bar{k}, i}$ is that $\dot{\pmb{v}}_{\bar{k}, i}^{(j)} = 0$ when $j\notin S_{\bar{k}}$, we have 
\begin{align}
    \nonumber
    \pmb{z}_k^{(j)} = &  \sum_{i=1}^p b_i \pmb{v}_i^{(j)} \pmb{v}_i^{\top}(\pmb{y}_k+\Delta \pmb{y}_k) \\
    \nonumber
    = &  \sum_{\bar{k}=1}^{\bar{K}} \sum_{i=1}^{p_{\bar{k}}} b_{\bar{k}, i} \dot{\pmb{v}}_{\bar{k}, i}^{(j)} \dot{\pmb{v}}_{\bar{k}, i}^{\top}(\pmb{y}_k+\Delta \pmb{y}_k) \\
    \nonumber
    = & \sum_{i}^{p_{\bar{k}_j}} b_{\bar{k}_j, i} \dot{\pmb{v}}_{\bar{k}_j, i}^{(j)} \dot{\pmb{v}}_{\bar{k}_j, i}^{\top}(\pmb{y}_k+\Delta \pmb{y}_k) \\
    \nonumber
    = & b_{\bar{k}_j,1} \dot{\pmb{v}}_{\bar{k}_j,1}^{(j)} \dot{\pmb{v}}_{\bar{k}_j,1}^{\top}(\pmb{y}_k+\Delta \pmb{y}_k) + \sum_{i=2}^{p_{\bar{k}_j}} b_{\bar{k}_j, i} \dot{\pmb{v}}_{\bar{k}_j, i}^{(j)} \dot{\pmb{v}}_{\bar{k}_j, i}^{\top}(\pmb{y}_k+\Delta \pmb{y}_k) \\
    \label{eq: zkj}
    = & \frac{1}{\beta+1} \dot{\pmb{v}}_{\bar{k}_j,1}^{(j)} \dot{\pmb{v}}_{\bar{k}_j,1}^{\top}(\pmb{y}_k+\Delta \pmb{y}_k) + \sum_{i=2}^{p_{\bar{k}_j}} b_{\bar{k}_j, i} \dot{\pmb{v}}_{\bar{k}_j, i}^{(j)} \dot{\pmb{v}}_{\bar{k}_j, i}^{\top}(\pmb{y}_k+\Delta \pmb{y}_k)
\end{align}
For the nonzero elements in $\dot{\pmb{v}}_{\bar{k}_j,1}$, if $j\in C_k$, there are at least $n_{\min}(1-\alpha)$ elements being $1$ at corresponding positions in $\pmb{y}_k+\Delta \pmb{y}_k$; if $j\notin C_k$, there are at most $n_{\max}\alpha c_{\max}$ elements being $1$ at corresponding positions in $\pmb{y}_k+\Delta \pmb{y}_k$.Then the inner product in the first term in equation \ref{eq: zkj} can be bounded by
\begin{align}
    \nonumber
\dot{\pmb{v}}_{\bar{k}_j,1}^{\top}(\pmb{y}_k+\Delta \pmb{y}_k) = \sum_{s=1}^{n_{\bar{k}_j}} \dot{\pmb{v}}_{\bar{k}_j,1}^{(s)}(\pmb{y}_k+\Delta \pmb{y}_k)^{(s)} 
\begin{cases}
\geq e_{\min,j}n_{\min}(1-\alpha), \quad j\in C_k \\
\leq e_{\max,j}\dot{\pmb{v}}_{\bar{k}_j,1}^{(j)}n_{\max}\alpha c_{\max}, \quad j\notin C_k 
\end{cases}
\end{align}
For the inner product in the second term in equation \ref{eq: zkj}, if $j\in C_k$, then  $\dot{\pmb{v}}_{\bar{k}_j,i}^{\top}(\pmb{y}_k+\Delta \pmb{y}_k)$ is the sum of at most $n_{\max}$ elements in $\dot{\pmb{v}}_{\bar{k}_j, i}$. Since $\|\dot{\pmb{v}}_{\bar{k}_j, i}\|_2^2 = 1$, the sum is bounded by $[-\sqrt{n_{\max}}, \sqrt{n_{\max}}]$. Similarly we can get the bound $[-\sqrt{n_{\max}c_{\max}}, \sqrt{n_{\max}c_{\max}}]$ for $j\notin C_k$. We also know that $|\dot{\pmb{v}}_{\bar{k}_j, i}^{(j)}|\leq 1$. Now it remains to bound $\sum_{i=2}^{p_{\bar{k}_j}}b_{\bar{k}_j,i}$ by applying Cauchy–Schwarz inequality and Corollary \ref{corollary:sum_smallest}
\begin{align}
    \nonumber
    \sum_{i=2}^{p_{\bar{k}_j}}b_{\bar{k}_j,i} = & \sum_{i=2}^{p_{\bar{k}_j}}\frac{\lambda_{\bar{k}, i}}{\beta+\lambda_{\bar{k}, i}} \\
    \nonumber
    = & p_{\bar{k}_j}-1 - \sum_{i=2}^{p_{\bar{k}_j}}\frac{\beta}{\beta+\lambda_{\bar{k}, i}}  \\
    \nonumber
    \leq & p_{\bar{k}_j}-1 - \frac{\beta(p_{\bar{k}_j}-1)^2}{\sum_{i=2}^{p_{\bar{k}_j}}(\lambda_{\bar{k}_j,i}+\beta)}\\
    \nonumber
    \leq & p_{\bar{k}_j}-1 - \frac{\beta(p_{\bar{k}_j}-1)^2}{2\sqrt{(p_{\bar{k}}-1)}\sqrt{\delta'}+\mathcal{O}(\delta')+\beta(p_{\bar{k}_j}-1)} \\
    \nonumber
    \leq & 2\sqrt{p-1}\frac{\sqrt{\delta'}}{\beta}+\mathcal{O}(\delta')
\end{align}
\end{proof}

\begin{lemma}\label{lemma: bound_e}
$e_{\min, j}$, $e_{\max, j}$ and $e_{max,j}/e_{\min, j}$ can be bounded.
\begin{align}
    \nonumber
    e_{\min, j} \geq & ~ \sqrt{\frac{1}{1+(n_{\max}-1)(1+\delta')^2}} \\
    \nonumber
    e_{\max, j } \leq & ~ \sqrt{\frac{1}{1+(n_{\min}-1)/(1+\delta')^2}} \\
    \nonumber
    \frac{e_{\max,j}}{e_{\min, j}} \leq & ~ (1+\delta')\sqrt{\frac{(1+(1+\delta')^2(n_{\max}-1))}{n_{\max}-1+(1+\delta')^2}} 
\end{align}
\end{lemma}
\begin{proof}
The proof is analogous to that for lemma \ref{lemma: perron_vec}.
\end{proof}

Now we are ready to calculate the maximum noise level that allows correct prediction on all training examples, i.e., $\pmb{z}_{k: j\in C_k}^{j} >  \pmb{z}_{k': j\notin C_{k'}}^{j}$. Let the preceding hold, then by lemma \ref{lemma: bound_z} we get
\begin{align}
    \label{eq: tolerance_with_e}
    \alpha < \frac{ 1-\frac{2\sqrt{n_{\max}}(1+\sqrt{c_{\max}})\sqrt{p-1}}{\frac{\beta}{\beta+1}e_{\min,j}\dot{\pmb{v}}_{\bar{k}_j,1}^{(j)}n_{\min}} \sqrt{\delta'}}{1+\frac{n_{\max} e_{\max, j}}{n_{\min} e_{\min, j}}c_{\max} } -\mathcal{O}(\delta')
\end{align}
Plugging lemma \ref{lemma: bound_e} into equation \ref{eq: tolerance_with_e} with some algebraic manipulation yields
\begin{align}
    \label{eq:alpha_rho}
    \alpha <  \frac{1}{1+\frac{n_{\max}}{n_{\min}}c_{\max}}- \mathcal{O}\left(\frac{\sqrt{\delta'}}{\beta}\right) = \frac{1}{1+\frac{n_{\max}}{n_{\min}}c_{\max}}- \mathcal{O}\left(\frac{\sqrt{\delta}}{\beta}\right).
\end{align}

\section{Considering Off-Diagonal Entries in the Adjacency Matrix (Connected Subclasses)}\label{sec: non_zero_off_diag}
For here on we assume $n_{\bar{k}} = \frac{n}{\bar{K}}, \forall \bar{k}\in\bar{K}$ for simplicity, despite that our results can easily extend to unbalanced dataset. 
\begin{lemma}\label{lemma: A_bar_with_off_diag}
Under assumption \ref{assump: block_with_off_diag}, the off-diagonal entries in $A$ is no longer zero. Let $\widetilde{\pmb{A}}$ denote the new normalized matrix, which also contains non-zero off-diagonal entries. With a bit abuse of notation, in the following we use $\overline{\pmb{A}}$ to denote the matrix obtained by normalizing $A$ with off-diagonal elements ignored. Then all the properties of eigenvectors and eigenvalues of $\overline{\pmb{A}}$ stated before including those lemma \ref{lemma: A_bar}, \ref{lemma: perron_vec}, \ref{lemma: bounded_squared_eigenv} still hold. And $\widetilde{\pmb{A}}$ can be written as a perturbation of $\overline{\pmb{A}}$
\begin{align}
    \nonumber
    \widetilde{\pmb{A}} = \overline{\pmb{A}} + \pmb{E},
\end{align}
with $\|\pmb{E}\|_F = \mathcal{O}\left( \bar{K}^{5/2}\xi \right)$.
\end{lemma}

\begin{proof}
Let $\mathring{\pmb{A}}$ be a matrix in the same shape of $\overline{\pmb{A}}$ containing all elements of $\widetilde{\pmb{A}}$ in the diagonal blocks. Let $\pmb{H}$ be a matrix that collects the remaining off-diagonal elements. Therefore $\widetilde{\pmb{A}} = \mathring{\pmb{A}} + \pmb{H}$, which can be rewritten as
\begin{align}
    \label{eq: decompose_A_tilde}
    \widetilde{\pmb{A}} = \overline{\pmb{A}} + \mathring{\pmb{A}} - \overline{\pmb{A}} + \pmb{H}.
\end{align}
For all off-diagonal elements $h_{i,j}$ in $\pmb{H}$
\begin{align}
    \nonumber
     h_{i,j} \leq \frac{\bar{K}\xi}{n}.
\end{align}
Since there are $n^2 - \sum_{\bar{k}=1}^{\bar{K}} n_{\bar{k}}^2$ entries outside of the diagonal blocks, the norm of $\pmb{H}$ can be bounded by
\begin{align}
    \label{eq: bound_H}
    \|\pmb{H}\|_F \leq \sqrt{1-\frac{1}{\bar{K}}}\bar{K}\xi
\end{align}
Each element in the diagonal blocks of $\overline{\pmb{A}}-\mathring{\pmb{A}}$ is non-negative. Also, supposing $\pmb{x}_i$ and $\pmb{x}_j$ are from subclass $\bar{k}$, we have
\begin{align}
    \nonumber
    (\overline{\pmb{A}}-\mathring{\pmb{A}})_{i,j} = & \frac{w_{\pmb{x}_i\pmb{x}_j}}{\sqrt{\sum_{s: \pmb{x}_s\in C_{\bar{k}} } w_{\pmb{x}_i\pmb{x}_s}}\sqrt{\sum_{t: \pmb{x}_t\in C_{\bar{k}}} w_{\pmb{x}_t\pmb{x}_j}}} - \frac{w_{\pmb{x}_i\pmb{x}_j}}{\sqrt{\sum_{s=1}^n w_{\pmb{x}_i\pmb{x}_s}}\sqrt{\sum_{t=1}^n w_{\pmb{x}_t\pmb{x}_j}}} \\
    \nonumber
    \leq & \frac{n(\bar{K}-1)\xi}{\frac{n}{\bar{K}(1+\delta)}(\frac{n}{\bar{K}(1+\delta)}+n(\bar{K}-1)\xi)}  \\
    \nonumber
    = & \mathcal{O}\left(\frac{\bar{K}^3(1+\delta)^2\xi}{n}\right),
\end{align}
by which the norm of $\overline{\pmb{A}}-\mathring{\pmb{A}}$ is bounded
\begin{align}
    \label{eq: bound_A_bar_A_ring}
    \|\overline{\pmb{A}}-\mathring{\pmb{A}}\|_F = \mathcal{O}\left( \bar{K}^{5/2}(1+\delta)^2\xi \right) = \mathcal{O}\left( \bar{K}^{5/2}\xi \right).
\end{align}
Combining equations \ref{eq: decompose_A_tilde}, \ref{eq: bound_H} and \ref{eq: bound_A_bar_A_ring} completes the proof.
\end{proof}

$\overline{\pmb{A}}$ has the following eigendecomposition 
\begin{align}
    \nonumber
    \overline{\pmb{A}} = 
    \begin{bmatrix}
    \pmb{V}_I \pmb{V}_N
    \end{bmatrix}
    \begin{bmatrix}
    \pmb{\Sigma}_I & 0 \\
    0 & \pmb{\Sigma}_N \\
    \end{bmatrix}
    \begin{bmatrix}
    \pmb{V}_I^{\top} \\
    \pmb{V}_N^{\top} \\
    \end{bmatrix},
\end{align}
where $\pmb{\Sigma}_I$ collects the $\bar{K}$ largest eigenvalues $\lambda_1, \dots, \lambda_{\bar{K}}$ on the diagonal and $\pmb{\Sigma}_N$ collects the remaining $\lambda_{\bar{K}+1}, \dots, \lambda_{\bar{n}}$. $\pmb{V}_I$ and $\pmb{V}_N$ collects the corresponding $\bar{K}$ and $n-\bar{K}$ eigenvectors, respectively. Let $\widetilde{\pmb{A}}$ has analogous decomposition
\begin{align}
    \nonumber
    \widetilde{\pmb{A}} = 
    \begin{bmatrix}
    \widetilde{\pmb{V}}_I \widetilde{\pmb{V}}_N
    \end{bmatrix}
    \begin{bmatrix}
    \widetilde{\pmb{\Sigma}}_I & 0 \\
    0 & \widetilde{\pmb{\Sigma}}_N \\
    \end{bmatrix}
    \begin{bmatrix}
    \widetilde{\pmb{V}}_I^{\top} \\
    \widetilde{\pmb{V}}_N^{\top} \\
    \end{bmatrix},
\end{align}
with eigenvalues $\widetilde{\lambda}_1, \dots, \widetilde{\lambda}_{n}$ and eigenvectors $\widetilde{\pmb{v}}_1, \dots, \widetilde{\pmb{v}}_n$. Eigenvalues of both matrices are indexed in descending order.

\subsection{Perturbation in Eigenvalues and Eigenvectors}\label{sec: perturb_eigen}
The following two lemmas bound the changes in eigenvalues, eigenvectos and the alignment between labels and eigenvectors caused by $\xi$.
\begin{lemma} \label{lemma: perturb_eigenvalue}
We have the following bound for eigenvalues of $\widetilde{\pmb{A}}$:
\begin{align}
\nonumber
    \begin{cases}
    ~~ 1-\mathcal{O}(K^{5/2}\xi) \leq \widetilde{\lambda}_i \leq 1, \quad & i = 1,2,\dots, \bar{K} \\
    ~~ \sum_{i=\bar{K}+1}^p \lambda_{i} \leq \mathcal{O}(\sqrt{\delta}+\bar{K}^{5/2}\xi)\\
    \end{cases}
\end{align}
\end{lemma}
\begin{proof}
From Lemma \ref{lemma:largest_egvalue} and Corollary \ref{corollary:sum_smallest} we know that 
\begin{align}
\label{eq: original_eigenvalue}
    \begin{cases}
    ~~ \lambda_i = 1, \quad & i = 1,2,\dots, \bar{K}, \\
    ~~ \sum_{i=\bar{K}+1}^p \lambda_{i} \leq \mathcal{O}(\sqrt{\delta}).\\
    \end{cases}
\end{align}
By Weyl’s inequality on perturbation, we have 
\begin{align}
\nonumber
    |\widetilde{\lambda}_i - \lambda_i| \leq \|\pmb{E}\|_2.
\end{align}
The right-hand-side is $\leq \|\pmb{E}\|_F$ and therefore $\mathcal{O}(\bar{K}^{5/2}\xi)$ by lemma \ref{lemma: A_bar_with_off_diag}. Combining the preceding with equation \ref{eq: original_eigenvalue} completes the proof.

\end{proof}

\begin{lemma}\label{lemma: perturb_proj}
The norm of the projection of $\pmb{Y}$ onto the range of $\widetilde{\pmb{V}}_I$ is bounded from below, i.e.,
\begin{align}
    \nonumber
    \|\widetilde{\pmb{V}}_I \widetilde{\pmb{V}}_I^T \pmb{Y}\|_F^2 \geq \|\pmb{V}_I \pmb{V}_I^T \pmb{Y}\|_F^2 - \mathcal{O}(n\bar{K}^2\xi).
\end{align}
\end{lemma}
\begin{proof}
By Lemma \ref{lemma:largest_egvalue} and Corollary \ref{corollary: bound_lambda_2} we have
\begin{align}
\nonumber
    \lambda_{\bar{K}} - \lambda_{\bar{K}+1} \geq 1-\mathcal{O}(\sqrt{\delta}).
\end{align}
By Wedin's Theorem \cite{wedin1972perturbation, stewart1990matrix}, we have the following bound on the principle angle between the range of $\widetilde{\pmb{V}}_I$ and the range of $\pmb{V}_I$ as long as $1\geq \mathcal{O}(\sqrt{\delta}+ \bar{K}^{5/2}\xi)$
\begin{align}
    \nonumber
    \| \pmb{V}_I \pmb{V}_I^T(\widetilde{\pmb{V}}_I \widetilde{\pmb{V}}_I^T- I)\|_F \leq \mathcal{O}\left(\frac{\|\pmb{E}\|_F}{\lambda_{\bar{K}}-\lambda_{\bar{K}+1}}\right) \leq \mathcal{O}\left(\frac{\mathcal{O}(\bar{K}^{5/2}\xi)}{1-\mathcal{O}(\sqrt{\delta})}\right) = \mathcal{O}(\bar{K}^{5/2}\xi).
\end{align}
Thus
\begin{align}
    \nonumber
    \|\widetilde{\pmb{V}}_I \widetilde{\pmb{V}}_I^T \pmb{Y}\|_F^2 = & \|\widetilde{\pmb{V}}_I \widetilde{\pmb{V}}_I^T \pmb{Y} - \pmb{V}_I \pmb{V}_I^T \pmb{Y} + \pmb{V}_I \pmb{V}_I^T \pmb{Y}\|_F^2 \\
    \nonumber
    \geq  & \|\pmb{V}_I \pmb{V}_I^T \pmb{Y}\|_F^2 + 2\langle \widetilde{\pmb{V}}_I \widetilde{\pmb{V}}_I^T \pmb{Y} - \pmb{V}_I \pmb{V}_I^T \pmb{Y}, ~~ \pmb{V}_I \pmb{V}_I^T \pmb{Y}  \rangle_F \\
    \nonumber
    \geq & \|\pmb{V}_I \pmb{V}_I^T \pmb{Y}\|_F^2 + 2\Tr((\widetilde{\pmb{V}}_I \widetilde{\pmb{V}}_I^T - \pmb{V}_I \pmb{V}_I^T) \pmb{Y} \pmb{Y}^T \pmb{V}_I \pmb{V}_I^T)\\
    \nonumber
    = & \|\pmb{V}_I \pmb{V}_I^T \pmb{Y}\|_F^2 + 2\Tr(\pmb{V}_I \pmb{V}_I^T(\widetilde{\pmb{V}}_I \widetilde{\pmb{V}}_I^T - I)\pmb{Y}\pmb{Y}^T) \\
    \nonumber
    = & \|\pmb{V}_I \pmb{V}_I^T \pmb{Y}\|_F^2 + 2\langle \pmb{V}_I \pmb{V}_I^T(\widetilde{\pmb{V}}_I \widetilde{\pmb{V}}_I^T - \pmb{I}), ~~\pmb{Y}\pmb{Y}^T\rangle_F \\
    \nonumber
    \geq & \|\pmb{V}_I \pmb{V}_I^T \pmb{Y}\|_F^2 - 2\| \pmb{V}_I \pmb{V}_I^T(\widetilde{\pmb{V}}_I \widetilde{\pmb{V}}_I^T - \pmb{I})\|_F \|\pmb{Y}\pmb{Y}^T\|_F \\
    \nonumber
    \geq & \|\pmb{V}_I \pmb{V}_I^T \pmb{Y}\|_F^2 - \mathcal{O}(\bar{K}^{5/2}\xi) \frac{n}{\sqrt{K}}\\
    \nonumber
    = & \|\pmb{V}_I \pmb{V}_I^T \pmb{Y}\|_F^2 - \mathcal{O}(n\bar{K}^2\xi).
\end{align}

\end{proof}
\subsection{Error under Gaussian Noise (Proof for Theorem \ref{theorem: gaussian_loss})}
Considering $\xi$, rewrite $\textbf{bias}^2$ as
\begin{align}
    \nonumber
    \textbf{bias}^2 = & 1+ \frac{1}{n}\sum_{i=1}^p\sum_{j=1}^{K}(\widetilde{b}_i^2-2\widetilde{b}_i)(\widetilde{\pmb{v}_i}^{\top}\pmb{y}_j)^2 \\
    \nonumber
    \leq &  1 -  \frac{1}{n}\sum_{i=1}^{\bar{K}}\sum_{j=1}^{K}(2\widetilde{b}_i-\widetilde{b}_i^2)(\widetilde{\pmb{v}_i}^{\top}\pmb{y}_j)^2\\
    \nonumber
    \leq &  1 -  \frac{1}{n}(2\widetilde{b}_{\bar{K}}-\widetilde{b}_{\bar{K}}^2)\sum_{i=1}^{\bar{K}}\sum_{j=1}^{K}(\widetilde{\pmb{v}_i}^{\top}\pmb{y}_j)^2 \\
    \label{eq:bias_rewrite}
    \leq &  1 -  \frac{1}{n}(2\widetilde{b}_{\bar{K}}-\widetilde{b}_{\bar{K}}^2)\|\widetilde{\pmb{V}}_I \widetilde{\pmb{V}}_I^T \pmb{Y}\|_F^2,
\end{align}
where $\widetilde{b}_i = \frac{\widetilde{\lambda_i}}{\widetilde{\lambda_i}+\beta}$. Also, Lemma \ref{lemma: perron_vec} gives us the lower bound for $\|\pmb{V}_I \pmb{V}_I^T \pmb{Y}\|_F^2$
\begin{align}
    \nonumber
    \|\pmb{V}_I \pmb{V}_I^T \pmb{Y}\|_F^2 = & \sum_{\bar{k}}^{\bar{K}}\|\pmb{v}_{\bar{k},1}\|_1^2 \\
    \label{eq: VVY}
    \geq & ~n-\mathcal{O}(\delta).
\end{align}
Combining lemma \ref{lemma: perturb_eigenvalue}, lemma \ref{lemma: perturb_proj}, equation \ref{eq:bias_rewrite} and equation \ref{eq: VVY} yields the bound for the bias
\begin{align}
    \nonumber
    \nonumber
    \textbf{bias}^2 = & (\frac{\beta}{1+\beta})^2 + \mathcal{O}(\delta+\xi).
\end{align}
We bound the variance in the same manner as in Section \ref{sec: error_block_adj_matrix} by applying Cauchy–Schwarz inequality, Corollary \ref{corollary:sum_smallest} and Lemma \ref{lemma: perturb_eigenvalue}
\begin{align}
\nonumber
    \textbf{variance} =& \frac{\sigma^2}{n}\sum_{i=1}^p \widetilde{b}_i^2 \\
    \nonumber
    = & \frac{\sigma^2}{n}\sum_{i=1}^{\bar{K}} \widetilde{b}_i^2 + \frac{\sigma^2}{n}\sum_{i=\bar{K}+1}^{p} \widetilde{b}_i^2  \\
    \nonumber
    \leq & \sigma^2\frac{\bar{K}}{n}(\frac{1}{\beta+1})^2+\frac{\sigma^2}{n}\sum_{i=\bar{K}+1}^{p} (1-\frac{\beta}{\widetilde{\lambda_i}+\beta}) \\
    \nonumber
    = & \sigma^2\frac{\bar{K}}{n}(\frac{1}{\beta+1})^2 + \frac{\sigma^2}{n}(p-\bar{K}) -\frac{\sigma^2}{n}\sum_{i=\bar{K}+1}^{p} \frac{\beta}{\widetilde{\lambda_i}+\beta} \\
    \nonumber
    \leq & \sigma^2\frac{\bar{K}}{n}(\frac{1}{\beta+1})^2 + \frac{\sigma^2}{n}(p-\bar{K})- \frac{\sigma^2}{n}\frac{\beta(p-\bar{K})^2}{\sum_{i=\bar{K}+1}^p(\widetilde{\lambda_i}+\beta)} \\
    \nonumber
    \leq &  \sigma^2\frac{\bar{K}}{n}(\frac{1}{\beta+1})^2 + \sigma^2\mathcal{O}(\frac{\sqrt{\delta}+\xi}{\beta}).
\end{align}

\section{Contrastive Learning Slightly Improves the Alignment Between Jacobian Matrix and Ground-truth Labels}\label{apx:alignment}
We compare the alignments between the clean label vector and the initial Jacobian matrix of (1) network pretrained using SimCLR for 1000 epochs, (2) network pretrained using SimCLR for 100 epochs and (3) randomly initialized network. $\pmb{y}\in\mathbb{R}^{nK}$ is the vector obtained by flattening the label matrix $\pmb{Y}$, i.e., concatenating the $n$ rows of $\pmb{Y}$. Let $\pmb{z}(\pmb{x}_i, \pmb{W})\in \pmb{R}^K$ be the output of the network given example $\pmb{x}_i$ and parameters $\pmb{W}\in \mathbb{R}^d$ (we see the parameters of the network as a vector). Then the Jacobian $\pmb{J}$ is defined as
\begin{align}
    \nonumber
    \pmb{J}(\pmb{W}) = \left[ \frac{\partial \pmb{z}(\pmb{x}_1, \pmb{W})}{\pmb{W}} \dots \frac{\partial \pmb{z}(\pmb{x}_n, \pmb{W})}{\pmb{W}} \right]^{\top}.
\end{align}
Note that $\frac{\partial \pmb{z}(\pmb{x}_i, \pmb{W})}{\pmb{W}} \in\mathbb{R}^{d \times K}$, therefore $\pmb{J}(\pmb{W})\in\mathbb{R}^{nK\times d}$. In table  \ref{tab: projection} $\Pi_{I}(\pmb{y})$ is the projection of $\pmb{y}$ onto the span of the $10$ singular vectors of $\pmb{J}(\pmb{W}_0)$ with larges singular values and $\Pi_{N}(\pmb{y})$ is the projection of $\pmb{y}$ onto the span of the remaining singular vectors. Interestingly, pretraining for more epochs leads to larger $\Pi_{I}(\pmb{y})$ and smaller $\Pi_{N}(\pmb{y})$ and therefore larger $\|\pmb{J}\pmb{J}^T\pmb{y}\|_F/\|\pmb{J}\pmb{J}^T\|_F$. How much this slight improvement in the alignment contributes to the robustness deserves further investigation.

\section{Training Only the Last Layer v.s. Training All Layers} \label{apx:frz_unfrz}
Figure \ref{fig: frz_vs_unfrz} compares the performance of training only the linear layer (i.e., with the encoder frozen) and fine-tuning all layers (i.e., with the encoder unfrozen). For both CIFAR-10 and CIFAR-100 we first pretrain a Res-Net 32 using SimCLR for 1000 epochs and using the Adam optimizer with a learning rate of $3 \times 10^{-4}$, a weight decay of $1 \times 10 ^{-6}$ and a batch size of 128. For downstream tasks, we use the SGD optimizer with a learning rate of $5 \times 10^{-3}$, a weight decay of $1\times 10^{-3}$, a batch size of 64. We see that in most cases fine-tuning achieves a higher test accuracy. However, finetuning will eventually overfit if trained for longer. Also, we note that training all layers is more likely to overfit, especially under large noise level (column 3 in figure \ref{fig: frz_vs_unfrz}).

\begin{figure*}[t]
    \centering
	\includegraphics[width=.24\textwidth]{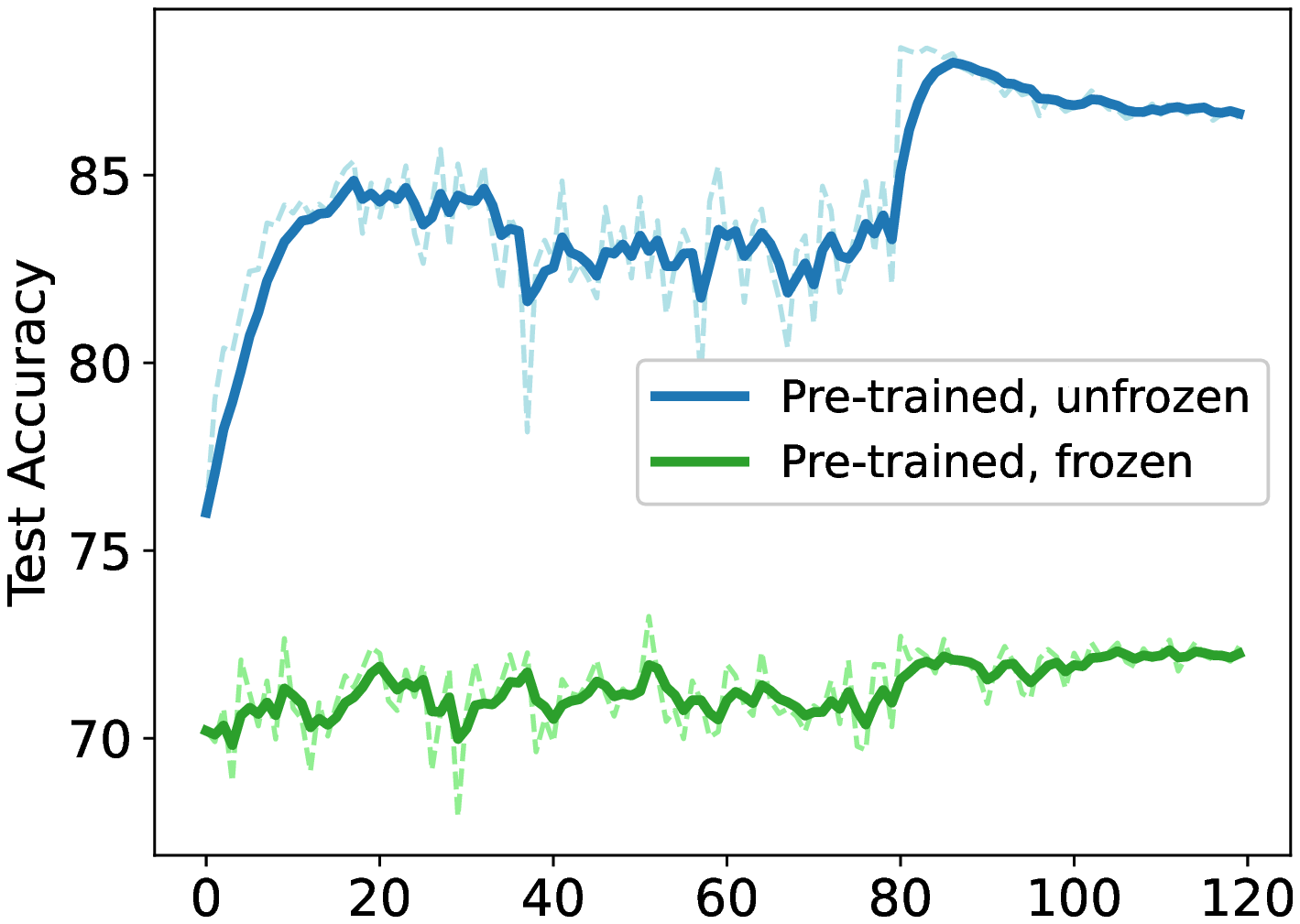}
	\includegraphics[width=.238\textwidth]{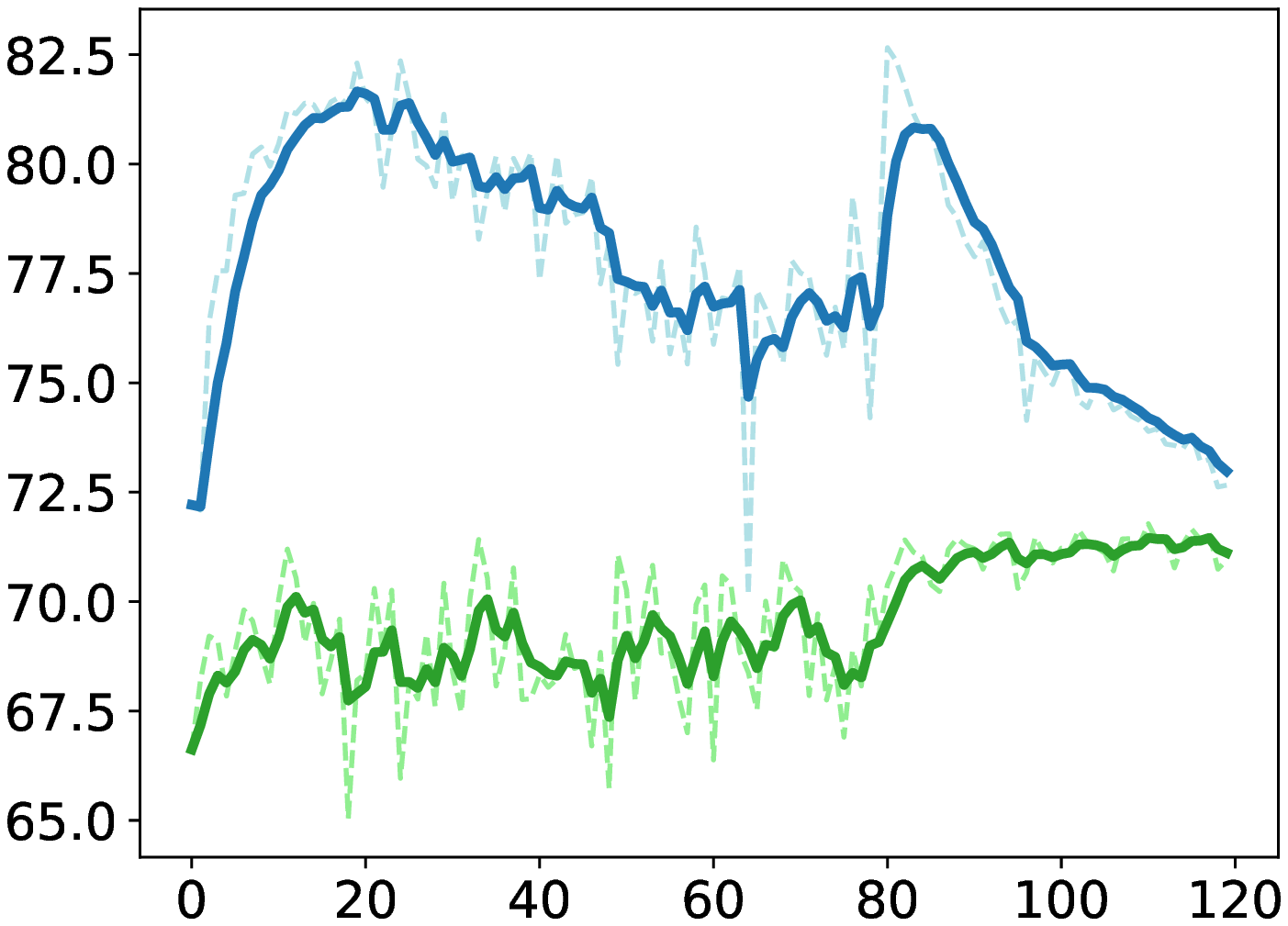}
	\includegraphics[width=.238\textwidth]{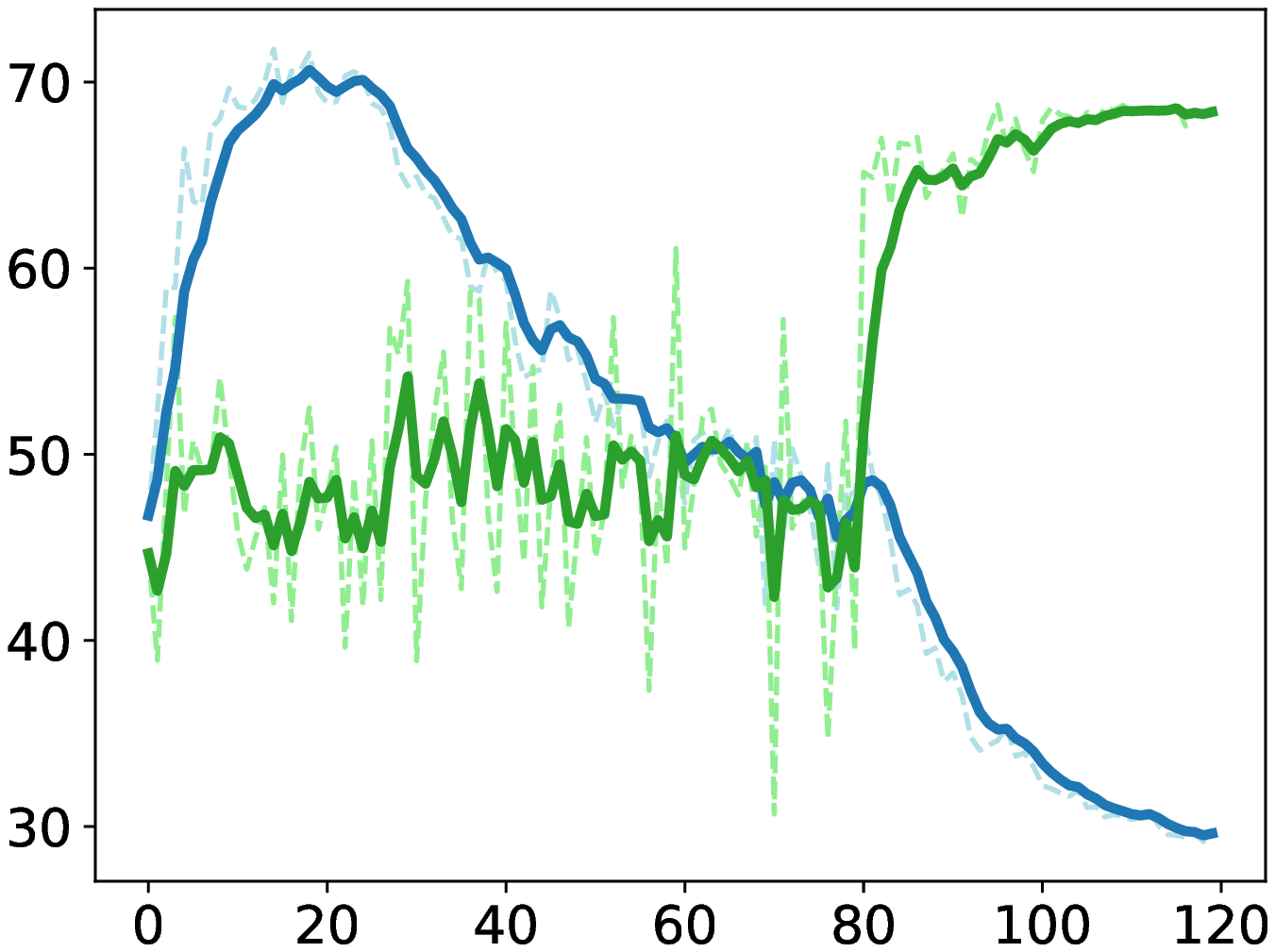}
	\includegraphics[width=.238\textwidth]{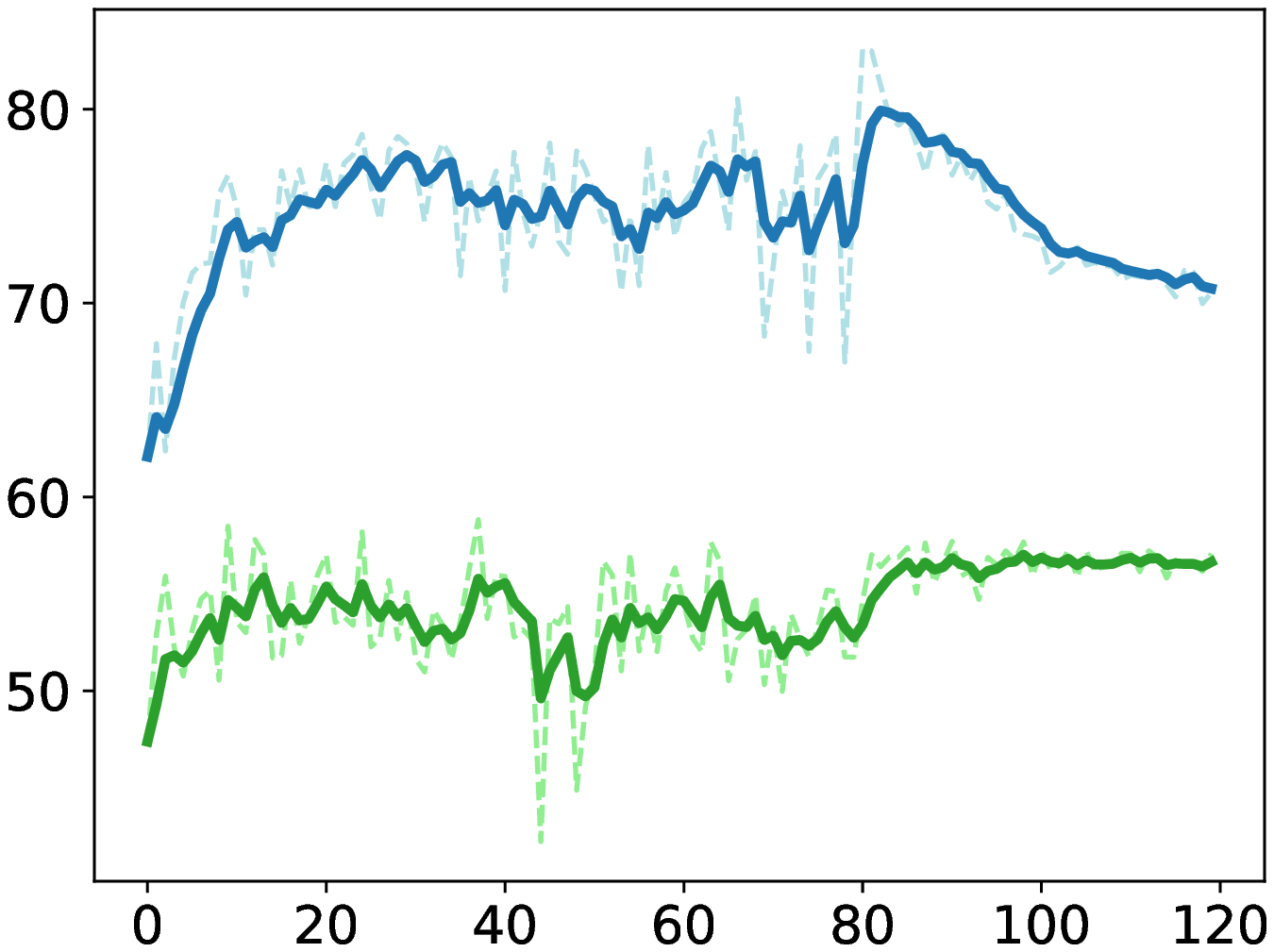}
	
	\includegraphics[width=.24\textwidth]{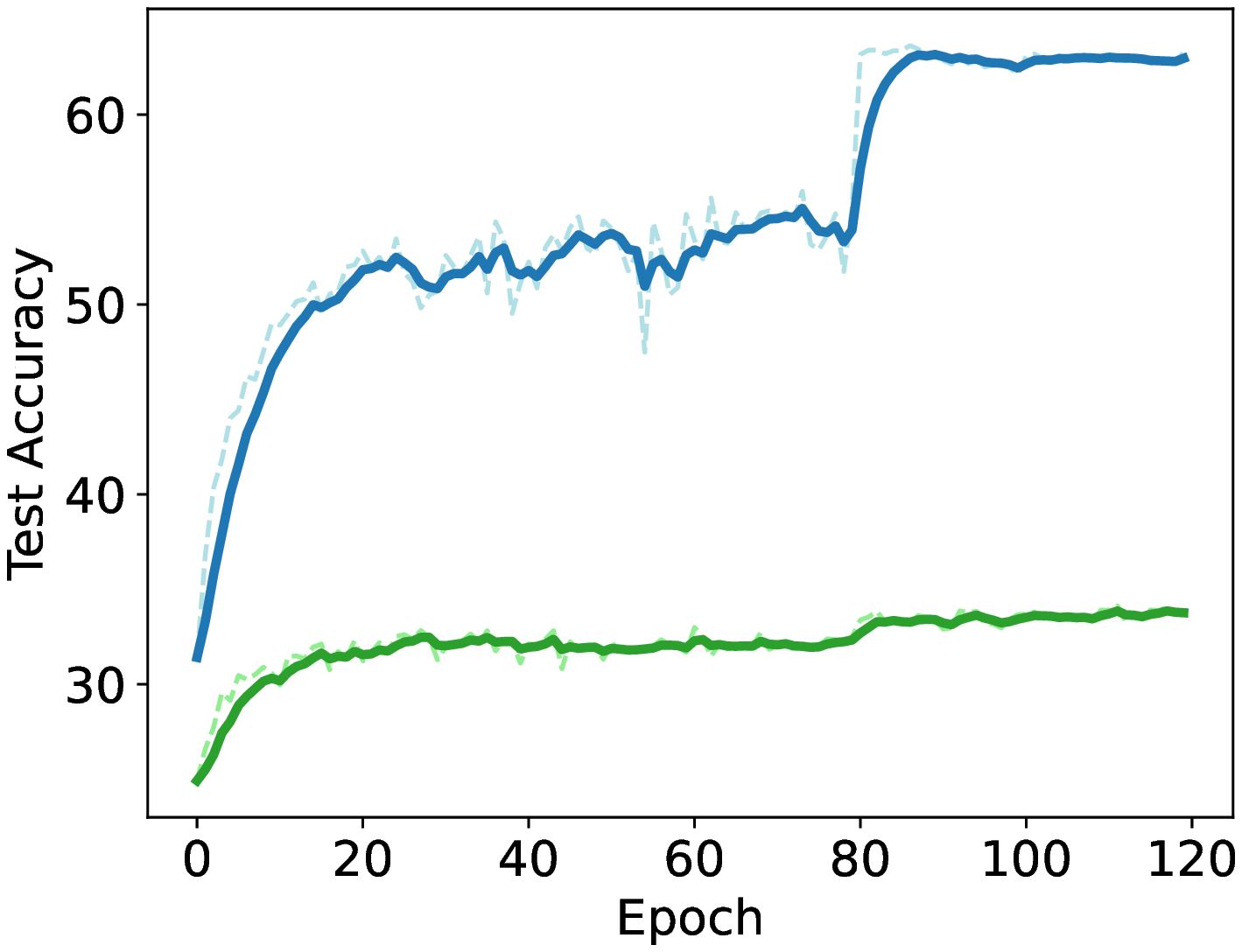}
	\includegraphics[width=.238\textwidth]{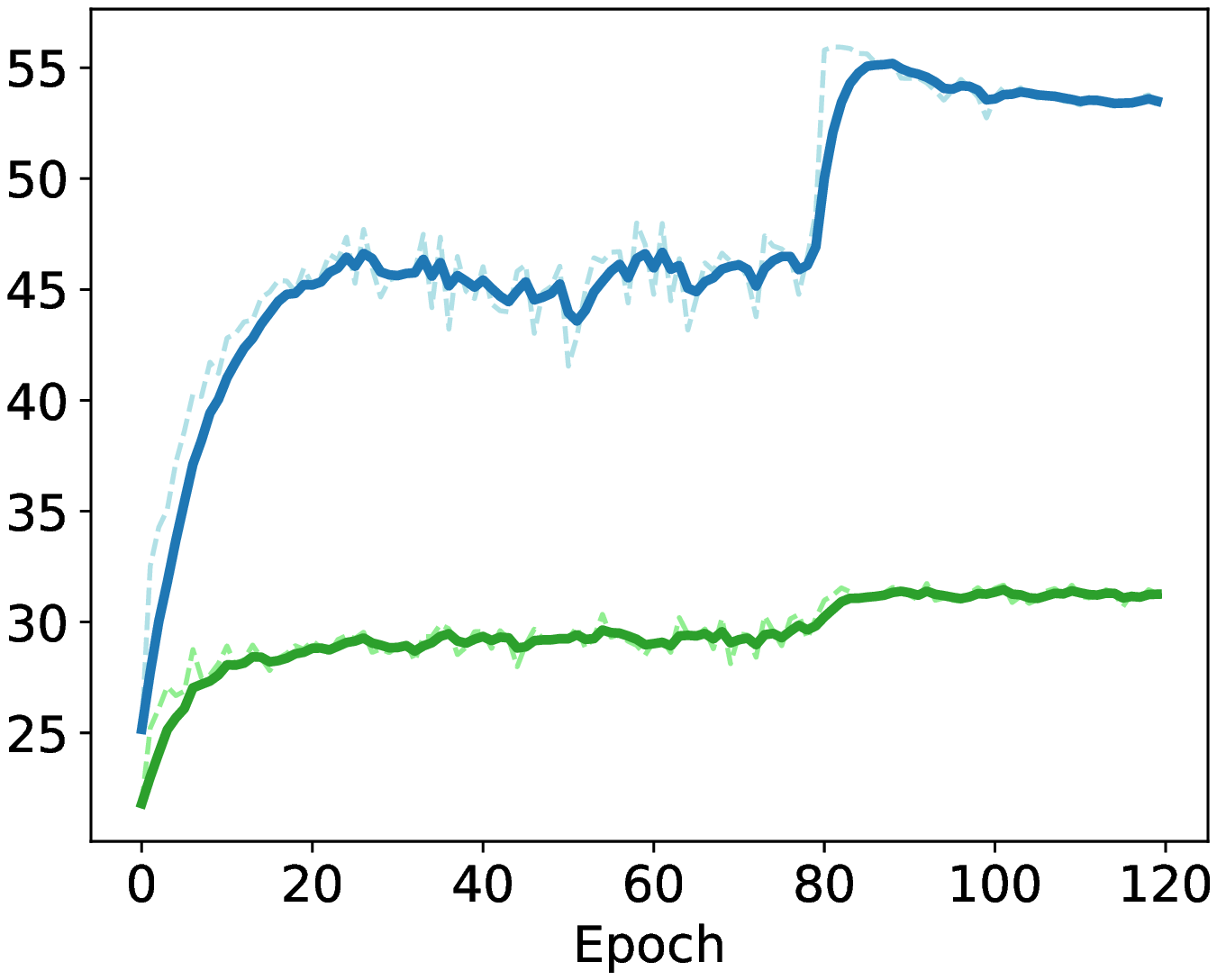}
	\includegraphics[width=.238\textwidth]{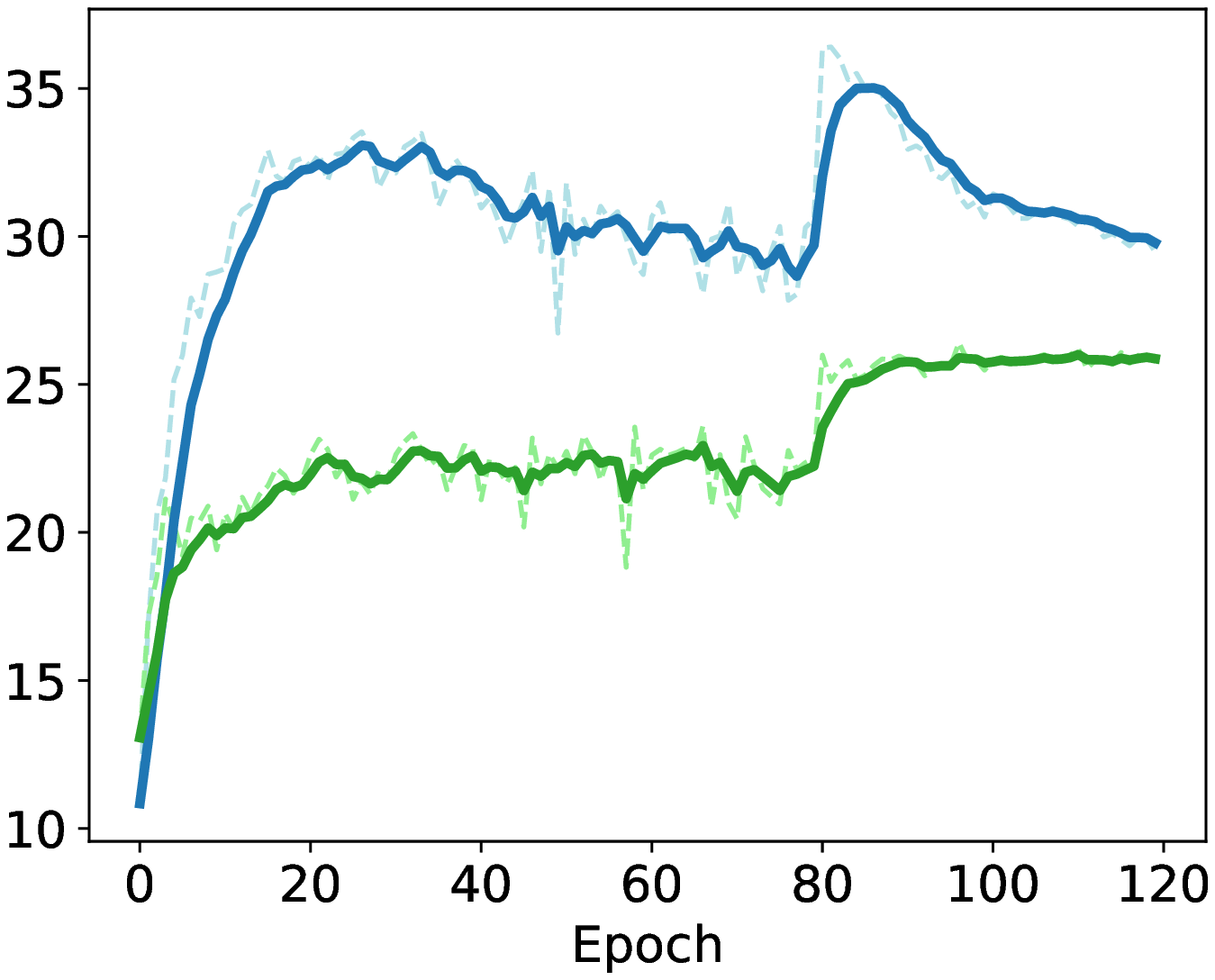}
	\includegraphics[width=.238\textwidth]{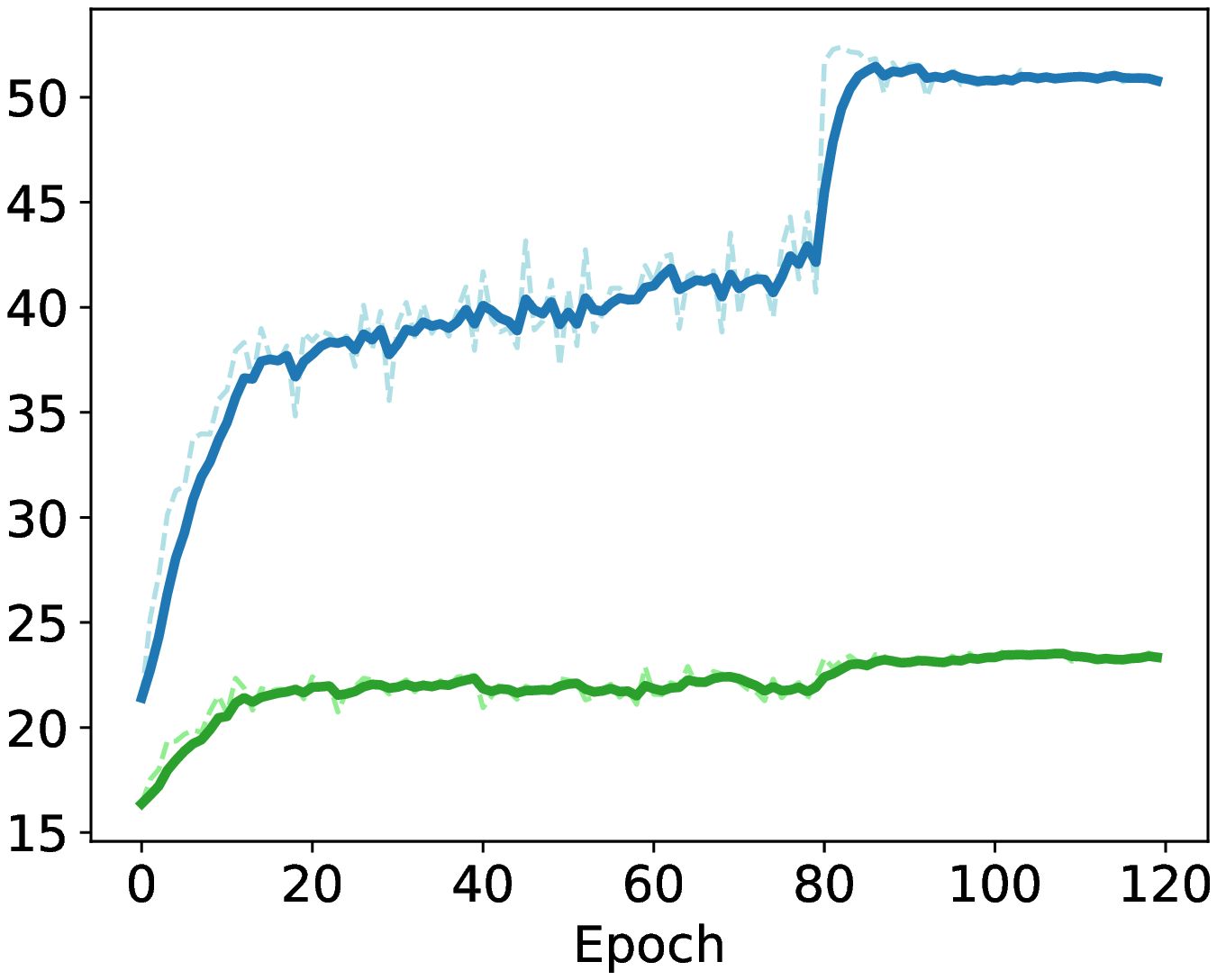}
    \caption{Test accuracy of fine-tuning a pre-trained network with frozen encoder v.s. unfrozen encoder on CIFAR-10 (top) and CIFAR-100 (bottom) under 20\%, 50\%, 80\% symmetric noise and 40\% asymmetric noise (left to right). }
    \label{fig: frz_vs_unfrz}
\end{figure*}

\begin{table}[t]
    \centering
    \begin{tabular}{|c|c|c|c|}
    \hline
         & $\|\Pi_{\mathcal{I}}(\pmb{y})\|_F$ & $\|\Pi_{\mathcal{N}}(\pmb{y})\|_F$ & $\|\pmb{J}\pmb{J}^T\pmb{y}\|_F/\|\pmb{J}\pmb{J}^T\|_F$ \\
         \hline
        Pretrained for 1000 epochs& 10.063 & 29.979 & 3.184 \\
        \hline
        Pretrained for 100 epochs & 10.036 & 29.988 & 3.175 \\
        \hline
        Randomly initialized & 10.014 & 29.995 & 3.055 \\
        \hline
    \end{tabular}
    \caption{Alignment between the Jacobian matrix and the clean labels.}
    \label{tab: projection}
\end{table}